\documentclass{article}

\usepackage{xcolor}

\usepackage{xcolor}
\usepackage{amsthm}
\usepackage{framed}
\theoremstyle{plain}

\newtheorem{definition}{Definition}
\newtheorem{theorem}{Theorem}

\newtheorem{lemma}{Lemma}

\newtheorem{prop}{Proposition}
\newtheorem{assumption}{Assumption}

\newcommand{\haochuan}[1]{\textcolor{blue}{\{{\textbf{HL}: \em #1}\}}}

\newcommand{\gr}{\nabla}

\newcommand{\E}{\mathbb{E}}
\newcommand{\pr}{\mathbb{P}}

\newcommand*\diff{\mathop{}\!\mathrm{d}}

\newcommand{\Spider}{\textsc{Spider}}
\newcommand{\GD}{\textsc{Gd}}
\newcommand{\SGD}{\textsc{Sgd}}
\newcommand{\SVRG}{\textsc{Svrg}}
\newcommand{\SARAH}{\textsc{Sarah}}

\usepackage[T1]{fontenc}
\usepackage[cmex10]{amsmath}
\usepackage[utf8]{inputenc}
\usepackage{pgfplots}
\makeatletter
\def\BState{\State\hskip-\ALG@thistlm}
\makeatother

\usepackage[T1]{fontenc}
\usepackage[english]{babel}
\usepackage[utf8]{inputenc}
\usepackage{pgfplots}
\usepackage{graphicx}
\usepackage{subcaption}
\usepackage{pstricks}
\usepackage{color}
\usepackage{amsmath}
\usepackage{bbm}
\usepackage{tcolorbox}
\usepackage{tikz}
\usepackage{pgfplots}
\usepackage{wrapfig}
\usepackage{lipsum}  
\usetikzlibrary{positioning}
\usepackage{tcolorbox}
\usepackage{amsfonts,amssymb}

\usepackage[round]{natbib}
\usepackage{mathtools}
\usepackage{commath}
\usepackage{relsize}
\usepackage{bbm}
\usepackage{bm}
\usepackage[font={small}]{caption}
\usepackage{comment,color,soul}
\usepackage{enumerate} 
\usetikzlibrary{arrows,automata}
\usepackage{amsfonts}
\usepackage{url}
\usepackage{lipsum}
\usepackage[thinlines]{easytable}
\usepackage{nicefrac}
\usepackage{algorithm}
\usepackage{algpseudocode}
\usepackage{mysymbol}
\usepackage{float}
\usepackage{titling}
\usepackage{comment}
\usepackage{scalerel}

\usepackage[top=2.5cm,bottom=3.5cm,left=2.5cm,right=2.5cm,marginparwidth=1.75cm]{geometry}

 \makeatletter
\def\BState{\State\hskip-\ALG@thistlm}
\makeatother

\makeatletter
\newcommand{\algmargin}{\the\ALG@thistlm}
\makeatother

\usepackage{amsmath}
\mathtoolsset{showonlyrefs}
\usepackage{etoolbox}

\makeatletter
\newif\ifdavid@number
\preto\equation{\david@numberfalse}
\preto\endequation{\ifdavid@number\else\notag\fi}
\patchcmd\label@in@display{\@empty}{\@empty\david@numbertrue}{}{}
\makeatother

\makeatletter
\def\Let@{\def\\{\notag\math@cr}}

\makeatother

\usepackage[hidelinks]{hyperref}
\definecolor{darkred}{RGB}{150,0,0}
\definecolor{darkgreen}{RGB}{0,150,0}
\definecolor{darkblue}{RGB}{0,0,150}
\hypersetup{colorlinks=true, linkcolor=darkred, citecolor=darkblue, urlcolor=darkblue}

\usepackage{enumitem}
\usepackage{amsfonts} 
\usepackage{tcolorbox}
\mathtoolsset{showonlyrefs}
\usepackage{float}
\usepackage{titling}
\usepackage{comment}
\usepackage{scalerel}
\usepackage{breqn}

\usepackage{tabulary,booktabs}
\usepackage{multirow}

\usepackage{pgfplots}
\usetikzlibrary{calc} 
\DeclareMathOperator*{\minimize}{minimize}

\usepackage{hhline}
\usepackage{tabu}
\usepackage{authblk}
\usepackage[symbol]{footmisc}

\title{\textbf{Variance-reduced Clipping for Non-convex Optimization}}

\date{}

\renewcommand{\thefootnote}{\fnsymbol{footnote}}

\renewcommand{\thefootnote}{\ifcase\value{footnote}\or*\or
1\or2\or(\#)\or(\#\#)\or(\#\#\#)\or(\#\#\#\#)\or($\infty$)\fi}
    
\author{Amirhossein Reisizadeh$^*$}
\author{Haochuan Li$^*$}
\author[2]{Subhro Das}
\author{Ali Jadbabaie}
\affil[1]{Massachusetts Institute of Technology}
\affil[2]{MIT–IBM Watson AI Lab, IBM Research}

\thanksmarkseries{arabic}
\begin{document}

\maketitle

\footnotetext[1]{Equal contribution.}

\renewcommand{\thefootnote}{\ifcase\value{footnote}\or
1\or2\or(\#)\or(\#\#)\or(\#\#\#)\or(\#\#\#\#)\or($\infty$)\fi}

\footnotetext[1]{\texttt{\{amirr,haochuan,jadbabai\}@mit.edu}}
\footnotetext[2]{\texttt{subhro.das@ibm.com}}

\begin{abstract}
Gradient clipping is a standard training technique used in deep learning applications such as large-scale language modeling to mitigate exploding gradients. Recent experimental studies have demonstrated a fairly special behavior in the smoothness of the training objective along its trajectory when trained with gradient clipping. That is, the smoothness \emph{grows} with the gradient norm. This is in clear contrast to the well-established assumption in folklore non-convex optimization, a.k.a. $L$--smoothness, where the smoothness is assumed to be bounded \emph{by a constant $L$ globally}. The recently introduced $(L_0,L_1)$--smoothness is a more relaxed notion that captures such behavior in non-convex optimization. In particular, it has been shown that under this relaxed smoothness assumption, {\SGD} with clipping requires $\ccalO(\epsilon^{-4})$ stochastic gradient computations to find an $\epsilon$--stationary solution. In this paper, we employ a variance reduction technique, namely {\Spider}, and demonstrate that for a carefully designed learning rate, this complexity is improved to $\ccalO(\epsilon^{-3})$ which is order-optimal. Our designed learning rate comprises the clipping technique to mitigate the growing smoothness. Moreover, when the objective function is the average of $n$ components, we improve the existing $\ccalO(n\epsilon^{-2})$ bound on the stochastic gradient complexity to $\ccalO(\sqrt{n} \epsilon^{-2} + n)$, which is order-optimal as well. In addition to being theoretically optimal, {\Spider} with our designed parameters demonstrates comparable empirical performance against variance-reduced methods such as {\SVRG} and {\SARAH} in several vision tasks.

\end{abstract}

\section{Introduction} \label{sec: intro}

We study the problem of minimizing a \emph{non-convex} function $F:\reals^d \to \reals$ which is expressed as the expectation of a stochastic function, \emph{i.e.},
\begin{align} \label{eq: stochastic}
    \minimize_{\bbx \in \reals^d} \, \,F(\bbx) = \E_{\xi}[f(\bbx; \xi)],
\end{align}
where the random variable $\xi$ is realized according to a distribution $\ccalD$. Typically, the distribution $\ccalD$ is unknown in this \emph{stochastic} setting, and rather, a number of realized samples are available. In this setting, known as \emph{finite-sum}, the objective $F$ can be expressed as the average of $n$ component functions $f_1, \cdots, f_n$, that is,
\begin{align} \label{eq: finite-sum}
    \minimize_{\bbx \in \reals^d} \, \,F(\bbx) = \frac{1}{n} \sum_{i=1}^{n} f_i(\bbx).
\end{align}
This formulation captures the standard training framework in many machine learning and deep learning applications where the model parameters are trained by minimizing the average loss induced by a large number of labeled data samples (also known as empirical risk minimization). 

\begin{figure}
\centering
\begin{subfigure}{.5\textwidth}
  \centering
    \includegraphics[width = .8\linewidth]{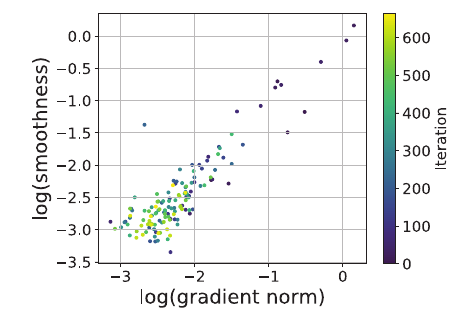}
    \caption{AWD-LSTM training  using gradient clipping}
    \label{fig: L0L1 smooth LSTM}
\end{subfigure}%
\begin{subfigure}{.5\textwidth}
  \centering
    \includegraphics[width = .66\linewidth]{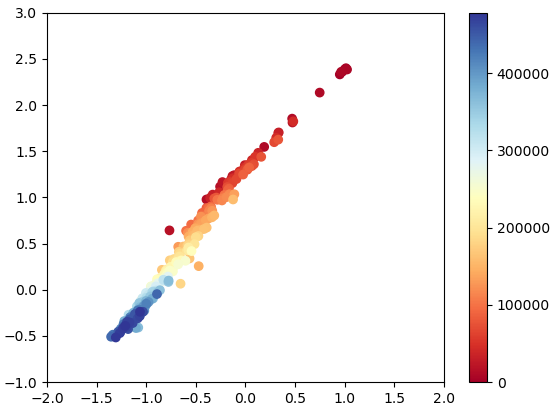}
    \vspace{5mm}
    \caption{Transformers training using Adam}
    \label{fig: L0L1 smooth WMT}
\end{subfigure}
\caption{Smoothness grows with gradient norm along the training trajectory for (a) AWD-LSTM on PTB dataset trained with clipped {\SGD} (Figure taken from \cite{zhang2019gradient}), and (b) transformers on WMT 2014 translation dataset trained with Adam (Figure taken from \cite{wang2022provable}).
}
\label{fig: L0L1 smooth}
\end{figure}

Gradient-based algorithms such as stochastic gradient descent ({\SGD}) have been widely used in training deep learning models due to their simplicity. However, adaptive gradient methods demonstrate superior performance over {\SGD} in particular applications such as natural language processing (NLP). In such methods, gradient clipping has been a standard practice in training language models to mitigate the exploding gradient problem. Although such superior performance of gradient clipping has not been well justified theoretically, \citep{zhang2019gradient} brings about some rationales to better understand the probable underpinning phenomenon. To bridge the theory-practice gap in this particular application, \citep{zhang2019gradient} first demonstrates an interesting characteristic of the optimization landscape of large-scale language models such as LSTM trained with gradient clipping. As illustrated in Figure \ref{fig: L0L1 smooth}, the smoothness of the objective \emph{grows} with the gradient norm along the training trajectory. This defies a well-established belief in smooth non-convex optimization where the smoothness is assumed to be bounded by a \emph{constant} over the input space, that is, $\Vert \gr^2 F(\bbx) \Vert \leq L$. Inspired by the experimental evidence, \citep{zhang2019gradient} introduces the more relaxed smoothness notion named $(L_0,L_1)$-smoothness, where the smoothness grows linearly with the gradient norm, that is, $\Vert \gr^2 F(\bbx) \Vert \leq L_0 + L_1 \Vert \gr F(\bbx) \Vert$ for positive constants $L_0,L_1$. This class of nonconvex functions includes many instances that do not have global Lipschitz gradients, such as the so-called exponential family. In particular, all polynomials of degree at least $3$, are $(L_0,L_1)$--smooth while there exists no constant bounding the smoothness globally \citep{zhang2019gradient}.

The standard performance measure of non-convex optimization algorithms is their required gradient computation to find \emph{approximate first-order stationary} solutions. More precisely, the goal of any non-convex optimization algorithm is to find a solution $\bbx$ such that 
\begin{align}
    \Vert \gr F(\bbx) \Vert \leq \epsilon,
\end{align}
for a given target accuracy $\epsilon$. The total number of gradient computations to find such a stationary point is defined as the \emph{first-order oracle} or \emph{gradient complexity} for the corresponding algorithm. Employing gradient clipping, \cite{zhang2019gradient} show that clipped {\SGD} is able to find an $\epsilon$--stationary solution of any $(L_0,L_1)$--smooth non-convex objective with gradient complexity at most $\ccalO(\epsilon^{-4})$. In this paper, we aim to answer the following question:
\begin{tcolorbox}
\begin{center}
\textit{Can we improve the gradient complexity $\ccalO(\epsilon^{-4})$ in $(L_0,L_1)$--smooth non-convex optimization?}
\end{center}
\end{tcolorbox}

We answer this question in the affirmative. In particular, we employ \emph{variance reduction} techniques and show that under regular conditions, the gradient complexity of finding an $\epsilon$--stationary point can be improved to $\ccalO(\epsilon^{-3})$ which is order-optimal.

Variance reduction has been a promising approach in speeding up non-convex optimization algorithms \emph{for $L$--smooth objectives}. Most relevant to our work is {\Spider} algorithm proposed in \citep{fang2018spider}. The core idea of {\Spider} is to devise and maintain an accurate estimator of the true gradient along the iterates. Let $\bbv_k$ denote the {\Spider}'s estimator for the true gradient $\gr F(\bbx_k)$ at iterate $k$. The gradient estimator $\bbv_k$ is updated in every iteration using a \emph{small} batch of stochastic gradients and the previous $\bbv_{k-1}$. To reset the undesired effect of stochastic gradients noise, a \emph{large} batch is used to update $\bbv_k$ once in a while. It has been shown that for a proper choice of the stepsize, {\Spider} is able to control the variance of the gradient estimator by $\epsilon^2$ for every iteration. Together with the standard Descent Lemma, \citep{fang2018spider} has shown that {\Spider} requires at most $\ccalO(\epsilon^{-3})$ stochastic gradient computations to reach an $\epsilon$--stationary point for any (averaged) $L$--smooth objective function. 
The adaptive stepsize corresponding to this order-optimal rate is picked as $\ccalO(\min\{1, \epsilon/\Vert \bbv_k \Vert\})$.

In this work, we aim to employ the variance reduction idea in {\Spider} to speed up the clipped {\SGD} algorithm for $(L_0,L_1)$--smooth objectives. The first challenge in doing so is that the standard Descent Lemma for $L$--smooth objectives does \emph{not} hold under the relaxed $(L_0,L_1)$--smooth condition. We show that the clipping component in the stepsize, \emph{i.e.} $\epsilon/\Vert \bbv_k \Vert$, equips us to establish a descent property under the new smoothness condition. 

The second and more critical challenge in utilizing the {\Spider} approach in our setting is to control the variance of the gradient estimator, that is $\E \Vert \bbv_k - \gr F(\bbx_k) \Vert^2$. As mentioned before, \cite{fang2018spider} show that for $L$--smooth objectives, the variance of the gradient estimator remains bounded by $\epsilon^2$ along the iterates if the stepsize is picked as $\ccalO(\min\{1, \epsilon/\Vert \bbv_k \Vert\})$. However, this stepsize does \emph{not} guarantee bounded variance in our $(L_0,L_1)$--smooth setting. In particular, we show that rather a \emph{smaller} stepsize is required to control the variance. That is, for stepsize $\ccalO(\min\{1, \epsilon/\Vert \bbv_k \Vert, \epsilon/\Vert \bbv_k \Vert^2 \})$, the variance of the gradient estimator $\bbv_k$ is provably controlled and bounded by $\epsilon^2$. The additional term $\epsilon/\Vert \bbv_k \Vert^2$ in the stepsize is essential in mitigating the growth of the smoothness which scales with the gradient norm in the relaxed  $(L_0,L_1)$--smooth setting. We shall refer to {\Spider} in this smoothness setup with this particular stepsize as $(L_0,L_1)$--{\Spider}.

Together with the (new) descent lemma, we show that $(L_0,L_1)$--{\Spider} with our devised choice of the learning rate described above, finds an $\epsilon$--stationary point of any non-convex and $(L_0,L_1)$--smooth function with high probability. More importantly, we demonstrate that the total  stochastic gradient computations required to find such a stationary point is at most $\ccalO(\epsilon^{-3})$. In our analysis, we relax the more restricted stochastic gradient assumption of almost surely bounded noise in \citep{zhang2019gradient} and impose a conventional and fairly generic assumption of bounded noise variance. In addition, we assume that the objective function is \emph{averaged} $(L_0,L_1)$--smooth over its stochastic components. Imposing the stronger averaged smoothness assumption on top of the weaker and typical one is indeed a standard practice in variance-reduced optimization literature. Under such generic assumptions, it has been shown that the $\Omega(\epsilon^{-3})$ rate is indeed a \emph{lower bound} on the stochastic gradient complexity for (averaged) $L$--smooth objectives \citep{fang2018spider}. Clearly, every $L$--smooth function is $(L,0)$--smooth, as well. Therefore, our $\ccalO(\epsilon^{-3})$ gradient complexity bound is also tight for the broader class of $(L_0,L_1)$--smooth non-convex functions.

We extend our results to the \emph{finite-sum} setting \eqref{eq: finite-sum} and show that for our devised pick of the stepsize $\ccalO(\min\{1, \epsilon/\Vert \bbv_k \Vert, \allowbreak\epsilon/\Vert \bbv_k \Vert^2 \})$, the variance reduction approach in {\Spider} is able to find an $\epsilon$--stationary point of any non-convex $(L_0,L_1)$--smooth function with high probability and at most $ \ccalO(\sqrt{n} \epsilon^{-2} + n)$ gradient computations. This significantly reduces the existing gradient complexity of clipped {\SGD} \citep{zhang2019gradient}, that is $ \ccalO(n \epsilon^{-2})$. In addition, the derived $ \ccalO(\sqrt{n} \epsilon^{-2} + n)$ complexity bound is naturally tight in the  $(L_0,L_1)$--smooth setting, as that is the case under the more restrictive $L$--smoothness condition. Tables \ref{table: rates} and \ref{table: lr} summarise our discussion.

\begin{table}[t]
\centering
\renewcommand{\arraystretch}{1.5}
\begin{tabular}{|cccc|}
\hline
Algorithm & Reference & Stochastic & Finite-sum\\ \hhline{|====|}
\textsc{ClippedSGD} & \cite{zhang2019gradient} & $\ccalO(\epsilon^{-4})$ & $\ccalO(n \epsilon^{-2})$ \\ 
\hline
$(L_0,L_1)$--{\Spider} & \textbf{This paper} & $\ccalO(\epsilon^{-3})$ & $ \ccalO(\sqrt{n} \epsilon^{-2} + n)$ \\
\hline
Lower bound & \cite{arjevani2022lower,fang2018spider}  & $\Omega(\epsilon^{-3})$ & $ \Omega(\sqrt{n} \epsilon^{-2})$$^\dagger$ \\
\hline
\end{tabular}
\caption{Complexity of $(L_0,L_1)$--smooth non-convex optimization. $^\dagger$This lower bound holds for $n \leq \ccalO(\epsilon^{-4})$. }
\label{table: rates}
\end{table}

\begin{table}[t]
\centering
{\tabulinesep=1.1mm
\begin{tabu}{|ccccc|}
\hline
Smoothness & Reference & Stochastic & Finite-sum & Learning rate\\ \hhline{|=====|}
$L$--smooth & \cite{fang2018spider} & $\ccalO(\epsilon^{-3})$ & $\ccalO(\sqrt{n} \epsilon^{-2} + n)$ & $\ccalO \Big( \min \Big\{ 1, \frac{\epsilon}{\Vert\bbv_k\Vert}\Big\}\Big)
$\\ 
\hline
$(L_0,L_1)$--smooth & \textbf{This paper}  & $\ccalO(\epsilon^{-3})$ & $\ccalO(\sqrt{n} \epsilon^{-2} + n)$ & $\ccalO \Big( \min \Big\{ 1, \frac{\epsilon}{\Vert\bbv_k\Vert}, \frac{\epsilon}{\Vert\bbv_k\Vert^2}\Big\} \Big)
$  \\
\hline
   \end{tabu}}
   \caption{Learning rates for {\Spider} and $(L_0,L_1)$--{\Spider}. All gradient complexities are order-optimal.}
\label{table: lr}
\end{table}

\textbf{Contributions.} To summarize the above discussion, we study non-convex optimization under $(L_0,L_1)$--smoothness and improve the existing gradient complexities of reaching first-order stationary solutions in both stochastic and finite-sum settings. We employ variance reduction technique {\Spider},
and devise learning rates resulting in order-optimal gradient complexities. Tables \ref{table: rates} compares this paper's results with the existing and optimal gradient complexities. In Table \ref{table: lr}, the learning rates resulting in order-optimal complexities in the folklore $L$--smooth and new $(L_0,L_1)$--smoothness settings are compared. Moreover, we implement {\Spider} with our designed learning rates and compare its empirical performance against several benchmarks such as {\SGD}, {\SARAH} and {\SVRG} tested on different image classification tasks with MNIST, CIFAR10 and CIFAR100 datasets.

\textbf{Notation.} Throughout the paper, we denote by $\Vert \bba \Vert$ the $\ell_2$--norm of vector $\bba$. We also let $\Vert \bbA \Vert$ denote the spectral norm of matrix $\bbA$. For non-negative functions $f,g : \ccalX \to [0, \infty)$ defined on the same domain, the standard big O notation $f = \ccalO(g)$ summarizes the fact that there exists a positive constant $c > 0$ such that $f(\bbx) \leq c \cdot g(\bbx)$ for all $\bbx \in \ccalX$. Moreover, we denote $f = \Omega(g)$ if there exists a positive constant $c > 0$ such that $f(\bbx) \geq c \cdot g(\bbx)$ for all $\bbx \in \ccalX$. Lastly, we use the shorthand notation $\bbx_{i:j}$ to denote the sequence $\bbx_i, \cdots, \bbx_j$.

\textbf{Related work.}

\textbf{$(L_0,L_1)$--smoothness and gradient clipping.} As mentioned before, gradient clipping has been widely used in training deep learning models such as large-scale language models to circumvent the exploding gradient challenge \citep{merity2017regularizing,gehring2017convolutional,peters2018deep}. The work of \cite{zhang2019gradient} lays out a theoretical framework to better understand the superior performance of clipped algorithms over the conventional non-adaptive gradient methods. 
Several follow-up works have studied the introduced $(L_0,L_1)$--smoothness notion by \citep{zhang2019gradient}. \citep{zhang2020improved} utilizes momentum techniques and sharpens the constant dependency of the convergence rate of clipped {\SGD} previously derived by \citep{zhang2019gradient}. Under this relaxed smoothness assumption, \citep{qian2021understanding} studies the role of clipping in incremental gradient methods. In the deterministic setting with full batch gradient computation, clipped {\GD} and normalized {\GD} (\textsc{NGD}) are essentially equivalent up to a constant. \citep{zhao2021convergence} provides convergence guarantees for \emph{stochastic} \textsc{NGD} for $(L_0,L_1)$--smooth non-convex functions. The work of \cite{faw2023beyond} builds on AdaGrad-type methods and relaxes the existing restrictive assumptions on the stochastic gradient noise. Clipping methods may be implemented with scalability \citep{liu2022communication} and privacy considerations \citep{yang2022normalized,xia2022differentially} as well. Under the same setting, it has been shown that unclipped methods, particularly a generalized \textsc{SignSGD} algorithm, attain the same rates as clipped {\SGD} \citep{crawshaw2022robustness}. Interestingly, this smoothness behavior is not restricted to the clipped {\SGD} optimizer as implemented by \cite{zhang2019gradient}; training language models with Adam manifests such smoothness phenomena as well \citep{wang2022provable}. Apart from the optimization literature, $(L_0,L_1)$--smoothness has been studied for variational inference problems as well \citep{sun2022convergence}.

\textbf{Variance reduction in non-convex optimization.} 
Variance reduction is known to be an effective approach for accelerating both convex and non-convex optimization algorithms. To recap the convergence rate improvement offered by variance reduction techniques in finding stationary points of non-convex functions with \emph{global} Lipschitz-gradients (\emph{i.e.} $L$--smooth), recall the folklore rate $\ccalO(\min\{n \epsilon^{-2}, \epsilon^{-4}\})$  of {\SGD}/{\GD} corresponding to finite-sum and stochastic  settings \citep{nesterov2003introductory}. Stochastic Variance-Reduced Gradient (SVRG) and Stochastically Controlled Stochastic Gradient (SCSG) improved the gradient complexity to $\tilde{\ccalO}(\min\{ n^{2/3}\epsilon^{-2},\epsilon^{-10/3}\})$ \citep{allen2016variance,reddi2016stochastic,lei2017non}. To further improve the gradient complexity, \citep{fang2018spider} introduced a more accurate and less costly approach to track the true gradients across the iterates, namely Stochastic Path-Integrated Differential EstimatoR ({\Spider}). This variance-reduced gradient method costs at most $\ccalO(\min\{\sqrt{n} \epsilon^{-2}, \epsilon^{-3}\})$ which matches the lower bound complexity in both the finite-sum \citep{fang2018spider} and the stochastic setting \citep{arjevani2022lower}. This makes {\Spider} an order-optimal algorithm to find stationary points of non-convex and smooth functions. Similarly and concurrently, SARAH was proposed \citep{nguyen2017stochastic} which shares the recursive stochastic gradient update framework with {\Spider}.  Moreover, \cite{zhou2020stochastic} proposed SNVRG with similar tight complexity bounds. Other works with order-optimal convergence rates include \citep{wang2019spiderboost,pham2020proxsarah,li2021page,li2021zerosarah}. These methods may be equipped with \emph{adaptive} learning rates such as \textsc{AdaSpider} \citep{kavis2022adaptive}.

\textbf{Variance reduction in deep learning.} Despite its promising theoretical advantages, variance-reduced methods demonstrate discouraging performance in accelerating the training of modern deep neural networks \citep{defazio2019ineffectiveness,defazio2014saga,roux2012stochastic,shalev2012stochastic}. As the cause of such a theory-practice gap remains unaddressed, there have been several speculations on the ineffectiveness of variance-reduced (and momentum) methods in deep neural network applications. It has been argued that the non-adaptive learning rate of such methods, e.g. SVRG could make the parameter tuning intractable \citep{cutkosky2019momentum}. In addition, as eluded in \citep{zhang2019gradient}, the misalignment of assumptions made in the theory and the practical ones could significantly contribute to this gap. Most of the variance-reduced methods described above heavily rely on the global Lipschitz-gradient assumption which has been observed not to be the case at least in modern NLP applications \citep{zhang2019gradient}.

\section{Preliminaries} \label{sec: preliminary}

In this section, we first review preliminary characteristics of $(L_0,L_1)$--smooth functions introduced in prior works and provide the assumption that we consider in the setting of this paper's interest, \emph{i.e.} stochastic and finite-sum. 

\subsection{$(L_0,L_1)$--smoothness}

As introduced in \citep{zhang2019gradient}, a function $F$ is said to be $(L_0,L_1)$--smooth if there exist constants $L_0 >0$ and $L_1 \geq 0$ such that for all $\bbx \in \reals^d$,
\begin{align} \label{eq: L0-L1 v1}
    \Vert \gr^2 F(\bbx) \Vert \leq L_0 + L_1 \Vert \gr F(\bbx) \Vert.
\end{align}
The twice-differentiability condition in this definition could be relaxed as noted in \citep{zhang2020improved} and stated below.
\begin{definition}[$(L_0,L_1)$--smooth] \label{def: L0-L1}
A differentiable function $F$ is said to be $(L_0,L_1)$--smooth if there exist constants $L_0 >0$ and $L_1 \geq 0$ such that if $\Vert \bbx - \bby \Vert \leq 1/L_1$, then
\begin{align} \label{eq: L0-L1 v2}
    \Vert \gr F(\bbx) - \gr F(\bby) \Vert
    \leq
    \big( L_0 + L_1 \Vert \gr F(\bbx) \Vert \big) \Vert \bbx - \bby \Vert.
\end{align}
\end{definition}
In the following, we show that these two conditions are essentially equivalent up to a constant. Therefore, moving forward, we set Definition \ref{def: L0-L1} as the main condition for $(L_0,L_1)$--smoothness.
\begin{prop} \label{prop: L0-L1}
If $F$ is twice differentiable, then condition \eqref{eq: L0-L1 v2} implies \eqref{eq: L0-L1 v1}. Moreover, condition \eqref{eq: L0-L1 v1} implies \eqref{eq: L0-L1 v2} with constants $(2L_0,2L_1)$.
\end{prop}

We defer the proof to Section \ref{sec: proof prop L0-L1}.
%
Definition \ref{def: L0-L1} states the smoothness condition on the main objective $F$. Clearly, this smoothness notion relaxes the traditional global Lipschitz-gradient assumption in non-convex optimization where for all $\bbx$ and $\bby$, $\Vert \gr F(\bbx) - \gr F(\bby) \Vert \leq L\Vert \bbx - \bby \Vert$ holds true; or, equivalently the function being twice differentiable, $\Vert \gr^2 F(\bbx) \Vert \leq L$ for all $\bbx$. In the setting of this paper's interest, the objective function $F$ is expressed as the (stochastic or finite-sum) average of component functions as formulated in \eqref{eq: stochastic} and \eqref{eq: finite-sum}. However, the smoothness condition \eqref{eq: L0-L1 v2} is solely imposed on the main objective $F$ irrespective of its components. As it is the standard assumption in variance-reduced optimization \citep{fang2018spider}, we impose the following \emph{averaged} smoothness condition of $F$ and its components.


\begin{assumption}[Averaged $(L_0,L_1)$--smooth] \label{assumption: L0-L1}
There exist constants $L_0 > 0$ and $ L_1 \geq 0$ such that if $\Vert \bbx - \bby \Vert \leq 1/L_1$, then
\begin{enumerate}[label=(\roman*)]
\item in the stochastic setting \eqref{eq: stochastic},
\begin{align}
    \E \Big[ \Vert \gr f(\bbx;\xi) - \gr f(\bby; \xi) \Vert^2 \Big]^{1/2}
    \leq
    \big( L_0 + L_1 \Vert \gr F(\bbx) \Vert \big) \Vert \bbx - \bby \Vert,
\end{align}
where the expectation is over random $\xi$; or,
\item in the finite-sum setting \eqref{eq: finite-sum},
\begin{align}
    \bigg( \frac{1}{n} \sum_{i=1}^{n} \Vert \gr f_i(\bbx) - \gr f_i(\bby) \Vert^2 \bigg)^{1/2} \leq
    \big( L_0 + L_1 \Vert \gr F(\bbx) \Vert \big) \Vert \bbx - \bby \Vert.
\end{align}
\end{enumerate}
\end{assumption}

It is worth noting that in both stochastic and finite-sum settings, the conditions in Assumption \ref{assumption: L0-L1} imply the smoothness of the main objective $F$ per Definition \ref{def: L0-L1}. Particularly in the finite-sum case, we have from Jensen's inequality that
\begin{align}
    \Vert \gr F(\bbx) - \gr F(\bby) \Vert
    \leq
    \E \Vert \gr f_i(\bbx) - \gr f_i(\bby) \Vert
    \leq
    \E \Big[ \Vert \gr f_i(\bbx) - \gr f_i(\bby) \Vert^2 \Big]^{1/2}
    \leq
    \big( L_0 + L_1 \Vert \gr F(\bbx) \Vert \big) \Vert \bbx - \bby \Vert,
\end{align}
where the last equality holds for any $\bbx,\bby$ such that $\Vert \bbx - \bby \Vert \leq 1/L_1$. A similar argument holds in the stochastic setting provided that the stochastic gradients are unbiased. Having set up the main smoothness assumption, we now review the existing gradient clipping methods and their convergence characteristics in the following section.

\subsection{Gradient clipping}

Similar to traditional (unclipped) gradient methods, a generic gradient clipping algorithm runs through iterations which we denote by $k=0,1,\cdots$, initialized with $\bbx_0$. At each iterate $k$, a stochastic gradient $\bbg_k \coloneqq \gr f(\bbx_k; \ccalS)$ is computed over a randomly selected mini-batch $\ccalS$ of size $|\ccalS|=S$. The iterate $\bbx_k$ is then updated as $\bbx_{k+1} = \bbx_k - \eta_k \bbg_k$ where $\eta_k$ denotes the learning rate (or stepsize). For instance and for a target accuracy $\epsilon$, the \emph{adaptive} stepsize can be expressed as $\eta_k = \ccalO( \min\{1, \epsilon/\Vert \bbg_k \Vert\})$ consisting of constant and clipping parts. Algorithm \ref{alg:SGD} summarizes this procedure which we denote by \textsc{ClippedSGD}. 
In the following, we illustrate the gradient complexity of this algorithm in finding $\epsilon$--stationary points of $(L_0,L_1)$--smooth functions. 

\begin{algorithm}[H]
\caption{\textsc{ClippedSGD}}\label{alg:SGD}
\begin{algorithmic}[1]
\State \textbf{Input:} smoothness parameters $L_0,L_1$, accuracy $\epsilon$, batchsize $S = |\ccalS|$, number of iterations $K$
\State Initialize $\bbx_0$
\For{$k=0,\cdots,K-1$}
\State Draw samples $\ccalS$ and compute $\bbg_k = \gr f(\bbx_k; \ccalS)$
\State Update $\bbx_{k+1} = \bbx_k - \eta_k \bbg_k$ \hfill{$\triangleright$ \, $\eta_k = \min\left\{ \dfrac{1}{2 L_0}, \dfrac{1}{L_0} \dfrac{\epsilon}{\norm{\bbg_k}}\right\}$
}
\EndFor
\State \textbf{return} $\tilde{\bbx}$ randomly and uniformly picked from $\{\bbx_0, \cdots, \bbx_{K-1}\}$
\end{algorithmic}
\end{algorithm}

Let us start with the stochastic setting \eqref{eq: stochastic}. We note that \citep{zhang2019gradient} characterizes the gradient complexity of \textsc{ClippedSGD} when the stochastic gradient noise is almost surely  bounded which we relax in our analysis. The following assumption precisely states the required condition on the stochastic gradient noise.
\begin{assumption} \label{assumption: stch gr}
Stochastic gradients $f(\cdot;\xi)$ are unbiased and variance-bounded, that is,
\begin{align}
    \E \big[ \gr f(\bbx;\xi) \big] = \gr F(\bbx),
    \quad \text{and} \quad
    \E \Vert \gr f(\bbx;\xi) - \gr F(\bbx) \Vert^2 \leq \sigma^2.
\end{align}
\end{assumption}

The above assumptions are standard and fairly general in stochastic optimization. We are now ready to state the iteration complexity of \textsc{ClippedSGD}.

\begin{theorem}[Stochastic setting] \label{thm: stochastic ClippedSGD}
Let Assumptions \ref{assumption: L0-L1} (i) and \ref{assumption: stch gr} hold and $\epsilon \leq \frac{L_0}{20L_1}$. Pick the stepsize and parameters below
\begin{align}
    \eta_k
    =
    \min\left\{ \frac{1}{2 L_0}, \frac{1}{L_0} \frac{\epsilon}{\norm{\bbg_k}}\right\},
    \quad
    S = \frac{\sigma^2}{\epsilon^2},
    \quad
    K = \bigg\lceil \frac{16 \Delta L_0}{\epsilon^2} \bigg\rceil.
\end{align}
Then, for the output of \textsc{ClippedSGD} in Algorithm \ref{alg:SGD}, i.e. $\tilde{\bbx}$ randomly and uniformly picked from $\{\bbx_{0:K-1}\}$, we have that $\Vert \gr F(\tilde{\bbx})\Vert \leq 12 \epsilon$ with probability at least $1/2$. Moreover, the total stochastic gradient complexity is bounded by $\Delta L_0 \sigma^2 \ccalO(\epsilon^{-4})$.
\end{theorem}

\begin{proof}
We defer the proof to Section \ref{sec: proof thm stochastic ClippedSGD}.
\end{proof}

A few remarks are in place. First, throughout the paper, we denote the initial suboptimality by $\Delta \coloneqq F(\bbx_0)  - \allowbreak F^*$  where the  global optimal $F^*$ is assumed to be a finite constant, that is, $F^* \coloneqq \min_{\bbx} F(\bbx) > -\infty$. Secondly, Theorem \ref{thm: stochastic ClippedSGD} still holds true with a relaxation of Assumptions \ref{assumption: L0-L1} (i) to condition \eqref{eq: L0-L1 v2} in Definition \ref{def: L0-L1}. Lastly, Theorem \ref{thm: stochastic ClippedSGD} improves the result of \cite{zhang2019gradient} (Theorem 7) and relaxes the almost sure bounded gradient noise to bounded variance.

In the finite-sum setting \eqref{eq: finite-sum}, \textsc{ClippedSGD} reduces to clipped {\GD} when the gradient $\bbg_k = \gr f(\bbx_k;\ccalS)$ is computed using the full batch of size $S = |\ccalS| = n$. \citep{zhang2019gradient} characterizes the iteration complexity of Clipped {\GD} and shows that to be bounded by $\ccalO(\Delta L_0 \epsilon^{-2} + \Delta L_1^2/L_0)$. For completeness of our presentation, we reproduce the convergence rate of \textsc{ClippedSGD} in this setting in the following.

\begin{theorem}[Finite-sum setting] \label{thm: finite-sum ClippedSGD}
Let Assumptions \ref{assumption: L0-L1} (ii) and \ref{assumption: stch gr} hold and $\epsilon \leq \frac{L_0}{20L_1}$. Pick the stepsize and parameters as below
\begin{align}
    \eta_k
    =
    \min\left\{ \frac{1}{2 L_0}, \frac{1}{L_0} \frac{\epsilon}{\norm{\bbg_k}}\right\},
    \quad
    S = n,
    \quad
    K = \bigg\lceil \frac{16 \Delta L_0}{\epsilon^2} \bigg\rceil.
\end{align}
Then,Then, for the output of \textsc{ClippedSGD} in Algorithm \ref{alg:SGD}, i.e. $\tilde{\bbx}$ randomly and uniformly picked from $\{\bbx_{0:K-1}\}$, we have that $\Vert \gr F(\tilde{\bbx})\Vert \leq 5 \epsilon$ with probability at least $1/2$. Moreover, the stochastic gradient complexity is bounded by $\ccalO(\Delta L_0 n \epsilon^{-2})$.
\end{theorem}

\begin{proof}
We defer the proof to Section \ref{sec: proof thm finite-sum ClippedSGD}.
\end{proof}

Theorems \ref{thm: stochastic ClippedSGD} and \ref{thm: finite-sum ClippedSGD} establish upper bounds on the gradient complexity of finding $\epsilon$--stationary points of non-convex and $(L_0,L_1)$--smooth functions which are $\ccalO(\epsilon^{-4})$ and $\ccalO(n \epsilon^{-2})$ for stochastic \eqref{eq: stochastic} and finite-sum \eqref{eq: finite-sum} optimization problems, respectively. As briefly eluded in Section \ref{sec: intro}, complexity lower bounds have been characterized for global Lipschitz-gradient objectives, \emph{i.e.} $L$--smooth. It has been shown that for given problem parameters, there exist $L$--smooth functions requiring at least $\ccalO(\epsilon^{-3})$ and $\ccalO(\sqrt{n}\epsilon^{-2} + n)$ stochastic gradient access to find $\epsilon$--stationary points \citep{arjevani2022lower,fang2018spider}. Since $L$--smooth functions are a subset of $(L_0,L_1)$--smooth ones, such lower bounds hold for the larger class. Therefore, we set our goal in the rest of the paper to address the following question:
\begin{center}
    \emph{Are lower bounds $\ccalO(\epsilon^{-3})$ and $\ccalO(\sqrt{n}\epsilon^{-2} + n)$ achievable for $(L_0,L_1)$--smooth non-convex functions, as well?} 
\end{center}

In the following section, we employ variance reduction techniques and formally prove that such lower bound rates are indeed achievable for the broader class of non-convex functions, \emph{i.e.} $(L_0,L_1)$--smooth.

\section{Main Results}

In this section, we present our main results and formally show that the gradient complexity lower bounds on finding stationary points of $(L_0,L_1)$--smooth functions mentioned in Section \ref{sec: preliminary} are achievable for both stochastic and finite-sum settings. In particular, we show how clipped and variance-reduced methods originally devised for $L$--smooth objectives can be adopted for the more general class of functions of our interest. Let us first review a variance-reduced method originally devised for non-convex objectives with global Lipschitz-gradient. 

\subsection{{\Spider}}

{\Spider} \citep{fang2018spider} is a first-order variance-reduced algorithm that efficiently finds stationary points of \emph{smooth} non-convex objectives with near-optimal gradient complexity (See Algorithm \ref{alg:Spider}). Particularly in the stochastic setting, it has been shown that when the instance functions $f(\cdot;\xi)$ have averaged $L$--Lipschitz gradients with $\sigma$--bounded variance, {\Spider} finds an $\epsilon$--stationary point of $F$ with (stochastic) gradient complexity of $\ccalO(L \sigma \epsilon^{-3})$. Similarly in the finite-sum setting, {\Spider} requires $\ccalO(L\sqrt{n}\epsilon^{-2} + n)$ gradient evaluations when the component functions $f_i$ have averaged $L$--Lipschitz gradients. In the following, we briefly describe the core idea of {\Spider}.

\begin{algorithm}[h!]
\caption{$(L_0,L_1)$--\Spider}\label{alg:Spider}
\begin{algorithmic}[1]
\State \textbf{Input:} smoothness parameters $L_0,L_1$, accuracy $\epsilon$, batchsizes $S_1=|\ccalS_1|,S_2 = |\ccalS_2|$, number of iterations $q,K$
\State Initialize $\bbx_0$
\For{$k=0,\cdots,K-1$}{
\If{$k \equiv 0 \mod q$}
\State Draw samples $\ccalS_1$ and compute $\bbv_k = \gr f(\bbx_k; \ccalS_1)$
\Else
\State Draw samples $\ccalS_2$ and compute $\bbv_k = \gr f(\bbx_k; \ccalS_2) - \gr f(\bbx_{k-1}; \ccalS_2) + \bbv_{k-1}$
\EndIf
}
\State Update $\bbx_{k+1} = \bbx_k - \eta_k \bbv_k$ \hfill{$\triangleright$ \, $\eta_k = \min\bigg\{ \dfrac{1}{2 L_0}, \dfrac{1}{L_0} \dfrac{\epsilon}{\norm{\bbv_k}}, \dfrac{1}{L_1} \dfrac{\epsilon}{\norm{\bbv_k}^2} \bigg\}$
}
\EndFor
\State \textbf{return} $\tilde{\bbx}$ randomly and uniformly picked from $\{\bbx_0, \cdots, \bbx_{K-1}\}$
\end{algorithmic}
\end{algorithm}

At every iteration $k$, {\Spider} (Algorithm \ref{alg:Spider}) maintains an estimate of the true gradient $\gr F(\bbx_k)$ denoted by $\bbv_k$ and updates it using a small batch of samples $\ccalS_2$ with size $S_2 = |\ccalS_2|$. Every $q$ iterations, $\bbv_k$ is refreshed with a large batch $\ccalS_1$ of size $S_1 = |\ccalS_1| \geq S_2$. In either case, the iterate is then updated as $\bbx_{k+1} = \bbx_k - \eta_k \bbv_k$. An important parameter in {\Spider} is the stepsize $\eta_k$. Originally and for $L$--smooth functions \citep{fang2018spider}, the learning rate is picked as  $\eta_k = \ccalO( \min\{L^{-1}, \epsilon L^{-1} / \Vert \bbv_k \Vert\})$. 
We adopt the core idea of {\Spider} for our broader smoothness setting and highlight the challenges in doing so as follows.

\textbf{Challenge.} In \citep{fang2018spider}, the authors show that for $L$--smooth objectives and proper parameters $S_1, S_2, q, K$ and the learning rate $\eta_k = \ccalO( \min\{1, \epsilon/\Vert \bbv_k \Vert\})$, {\Spider} (Algorithm \ref{alg:Spider}) is able to control the variance of the estimator $\bbv_k$ in every iteration. More precisely, it holds that $\E \Vert \bbv_k - \gr F(\bbx_k) \Vert^2 \leq \epsilon^2$. This is central to the near-optimal convergence rate of {\Spider}. However, when the $L$--smoothness assumption is relaxed to the $(L_0,L_1)$--smoothness, the same stepsize as before would not work under the same conditions. In particular, we show that a smaller learning rate $\ccalO(\min\{1, \epsilon/\Vert \bbv_k \Vert, \allowbreak\epsilon/\Vert \bbv_k \Vert^2 \})$ maintains the estimator's variance by $\epsilon^2$. We refer to {\Spider} method with this particular pick for the learning rate as $(L_0,L_1)$--{\Spider}. We defer further details to the proof sketch and the appendices.

In the following, we present our main convergence results for both stochastic and finite-sum cases followed by a sketch of the proof.

\subsection{Stochastic setting}

The next theorem characterizes an upper bound on the gradient complexity of finding $\epsilon$--stationary solutions in the stochastic setting \eqref{eq: stochastic}.


\begin{theorem}[Stochastic setting] \label{thm: stochastic spider}
Consider the stochastic minimization in \eqref{eq: stochastic} and let Assumptions \ref{assumption: L0-L1} (i) and \ref{assumption: stch gr} hold. Moreover, assume that $\epsilon < \frac{L_0}{20L_1}$ and pick the stepsize and parameters below
\begin{gather}
    \eta_k
    =
    \min\bigg\{ \frac{1}{2 L_0}, \frac{1}{L_0} \frac{\epsilon}{\norm{\bbv_k}}, \frac{1}{L_1} \frac{\epsilon}{\norm{\bbv_k}^2} \bigg\},
    \quad
    S_1 = \frac{4 \sigma^2}{\epsilon^2},
    \quad
    S_2 = 48 \frac{L_0}{L_1} \frac{\sigma}{\epsilon},
    \quad
    q = 2 \frac{L_0}{L_1} \frac{\sigma}{\epsilon},
    \quad
    K = \bigg\lceil \frac{16 \Delta L_0}{\epsilon^2} \bigg\rceil.
\end{gather}

Then, for the output of $(L_0,L_1)$--{\Spider} in Algorithm \ref{alg:Spider}, i.e. $\tilde{\bbx}$ randomly and uniformly picked from $\{\bbx_{0:K-1}\}$, we have that $\Vert \gr F(\tilde{\bbx})\Vert \leq 24 \epsilon$ with probability at least $1/2$. In addition, the stochastic gradient complexity is bounded by
\begin{align} 
    32 \Delta \sigma \left(L_1 + 24 \frac{L_0^2}{L_1} \right) \frac{1}{\epsilon^3}
    +
    \frac{4 \sigma^2}{\epsilon^2}
    +
    2 \sigma \left( \frac{L_1}{L_0} + 24 \frac{L_0}{L_1}\right) \frac{1}{\epsilon}.
\end{align}
\end{theorem}

\begin{proof}
We defer the proof to Section \ref{sec: proof thm stochastic spider}.
\end{proof}

Theorem \ref{thm: stochastic spider} shows that the variance reduction technique in the {\Spider} algorithm along with the prescribed choices of the parameters and the learning rate improves the gradient complexity $\ccalO(\epsilon^{-4})$ of \textsc{ClippedSGD} characterized in Theorem \ref{thm: stochastic ClippedSGD} to $\ccalO(\epsilon^{-3})$. It is also worth noting that the probability guarantee of $1/2$ provided by Theorem \ref{thm: stochastic spider} can be improved to any (constant) probability $1-p$ which in turn results in larger stochastic gradient complexity of $\ccalO(\epsilon^{-3}/\text{poly}(p))$.



\subsection{Finite-sum setting}
Next, we consider the finite-sum setting \eqref{eq: finite-sum} and provide the convergence guarantees for {\Spider} to find stationary points. 


\begin{theorem}[Finite-sum setting] \label{thm: finite-sum spider}
Consider the finite-sum minimization in \eqref{eq: finite-sum} and let Assumption \ref{assumption: L0-L1} (ii) hold. Furthermore, assume that $\epsilon < \frac{L_0}{20L_1}$ and pick the stepsize and parameters below
\begin{gather}
    \eta_k
    =
    \min\bigg\{ \frac{1}{2 L_0}, \frac{1}{L_0} \frac{\epsilon}{\norm{\bbv_k}}, \frac{1}{L_1} \frac{\epsilon}{\norm{\bbv_k}^2} \bigg\},
    \quad
    S_1 = n,
    \quad
    S_2 = 12 \sqrt{n},
    \quad
    q = \sqrt{n},
    \quad
    K = \bigg\lceil \frac{16 \Delta L_0}{\epsilon^2} \bigg\rceil.
\end{gather}
Then, for the output of $(L_0,L_1)$--{\Spider} in Algorithm \ref{alg:Spider}, i.e. $\tilde{\bbx}$ randomly and uniformly picked from $\{\bbx_{0:K-1}\}$, we have that $\Vert \gr F(\tilde{\bbx})\Vert \leq 24 \epsilon$ with probability at least $1/2$. Moreover, the stochastic gradient complexity of finding such a stationary point is bounded by
\begin{align} 
    208 \Delta L_0 \sqrt{n} \frac{1}{\epsilon^2}
    +
    n
    +
    13 \sqrt{n}.
\end{align}
\end{theorem}

\begin{proof}
We defer the proof to Section \ref{sec: proof thm finite-sum spider}.
\end{proof}


Considering the dominant terms, Theorem \ref{thm: finite-sum spider} improves the gradient complexity of \textsc{ClippedSGD} from  $\ccalO(n\epsilon^{-2})$ in Theorem \ref{thm: finite-sum ClippedSGD} to $\ccalO(\sqrt{n}\epsilon^{-2} + n)$. Similar to Theorem \ref{thm: stochastic spider}, the probability guarantee of $1/2$ can be improved to any probability $1-p$ with a larger stochastic gradient complexity of $\ccalO(\sqrt{n}\epsilon^{-2}/\text{poly}(p) + n)$. It is also worth noting that unlike Theorem \ref{thm: stochastic spider}, no bounded stochastic gradient noise such as Assumption \ref{assumption: stch gr} is required in Theorem \ref{thm: finite-sum spider}.


\textbf{Proof sketch.} To prove the convergence rates of Theorems \ref{thm: stochastic spider} and \ref{thm: finite-sum spider}, we establish two arguments which are summarized in the following for the stochastic setting. The finite-sum case follows from similar arguments and we defer the details to the appendices.

\begin{lemma}[Proof sketch] \label{lemma: sketch}
Consider the setup and parameters as stated in Theorem \ref{thm: stochastic spider} with $\epsilon < \frac{L_0}{20L_1}$. Then, for any iteration $k=0,1,\cdots$, we have that
\begin{flalign} \label{eq: descent}
    && F(\bbx_{k+1})
    \leq
    F(\bbx_k) - 
    \frac{1}{8} \eta_k \norm{\bbv_k}^2
    +
    \frac{5}{8} \eta_k \norm{\bbv_k - \gr F(\bbx_k)}^2, 
    &&\text{{\normalfont (descent inequality)}}\quad&
\end{flalign}
and
\begin{flalign} 
    &&\hspace{3cm} \E \Big[ \norm{\bbv_k - \gr F(\bbx_k)}^2 \Big] \leq \epsilon^2.
    &&\text{{\normalfont (bounded estimator's variance)}}\quad&
    \label{eq: bounded var}
\end{flalign}
\end{lemma}
The first argument \eqref{eq: descent} established a \emph{descent inequality} which is often essential to convergence proof of non-convex methods. Note that the folklore decent lemma breaks down in the $(L_0,L_1)$--smooth setting. Moreover, \eqref{eq: bounded var} guarantees a bounded and small error in estimating the true gradient in all iterations when the learning rate is picked as prescribed by Theorem \ref{thm: stochastic spider}. The finite-sum case in Theorem \ref{thm: finite-sum spider} follows from the same logic.

\textbf{Lower bounds.} To argue the optimality of the convergence rates derived in Theorems \ref{thm: stochastic spider} and \ref{thm: finite-sum spider}, we need to characterize the lower bounds on the gradient complexity of finding stationary solutions under the same conditions. Let us first consider the finite-sum setting and Theorem \ref{thm: finite-sum spider}. Under the $L$--smoothness setting, it has been shown that for any $L>0$ and $n \leq \ccalO(\epsilon^{-4})$, there exist a function $F$ of the form \eqref{eq: finite-sum} such that
\begin{align} 
    \E \Big[\Vert \gr f_i(\bbx) - \gr f_i(\bby) \Vert^2\Big]^{1/2} 
    \leq 
    L \Vert \bbx - \bby \Vert,
\end{align}
for which finding an $\epsilon$--stationary solution costs at least $\Omega(\sqrt{n} \epsilon^{-2})$ stochastic gradient accesses (\cite{fang2018spider}, Theorem 3). Clearly, any such function is also $(L,0)$--smooth per Definition \ref{def: L0-L1} and satisfies Assumption \ref{assumption: L0-L1} (ii) with $L_0=L$ and $L_1=0$. As a result, the gradient complexity of {\Spider} in Theorem \ref{thm: finite-sum spider} is order-optimal for $n \leq \ccalO(\epsilon^{-4})$, that is, $\ccalO(\sqrt{n}\epsilon^{-2} + n) = \ccalO(\sqrt{n}\epsilon^{-2})$ matching the lower bound $\Omega(\sqrt{n} \epsilon^{-2})$.

Similarly for the $L$--smooth stochastic setting, \citep[Theorem~2]{arjevani2022lower} shows that for any $L>0$ and $\sigma > 0$, there exists a function $F$ of the form \eqref{eq: stochastic} with stochastic gradients $g(\cdot;\xi)$ such that
\begin{align}
    \E \big[ g(\bbx;\xi) \big] = \gr F(\bbx),
    \quad \quad
    \E \Big[ \Vert g(\bbx;\xi) - \gr F(\bbx) \Vert^2\Big] \leq \sigma^2,
    \quad \text{and} \quad
    \E \Big[\Vert g(\bbx;\xi) - g(\bby; \xi) \Vert^2 \Big]^{1/2} 
    \leq 
    L \Vert \bbx - \bby \Vert,
\end{align}
for which finding an $\epsilon$--stationary solution requires at least $\Omega(\sigma \epsilon^{-3} + \sigma^2 \epsilon^{-2})$ stochastic gradient queries. Therefore, there exist $(L,0)$--smooth functions satisfying Assumptions \ref{assumption: L0-L1} (i) and \ref{assumption: stch gr} that cost at least $\Omega(\sigma \epsilon^{-3} + \sigma^2 \epsilon^{-2})$ stochastic gradient accesses making the {\Spider}'s rate in Theorem \ref{thm: stochastic spider} order-optimal.

\section{Experiments}

In this section, we empirically show the convergence behaviors of different variance reduction methods on various image classification tasks with neural networks. We have provided the code of our implementation in the following GitHub link\footnote{\href{https://github.com/haochuan-mit/varaince-reduced-clipping-for-non-convex-optimization}{\texttt{github.com/haochuan-mit/varaince-reduced-clipping-for-non-convex-optimization}}}.

\textbf{Models, Datases and Benchmarks:} We train three different neural network models: a three-layer fully connected network (FCN), ResNet-20 and ResNet-56 \citep{he2016deep} on three standard datasets for image classification: MNIST \citep{LeCunCortesBurgesMNIST}, CIFAR10 and CIFAR100. In every experiment, we compare $(L_0,L_1)$--{\Spider} against the relevant benchmarks approaches, namely \SGD, {\SVRG} \citep{reddi2016stochastic}, {\SARAH} \citep{li2021zerosarah}, and {\Spider} \citep{fang2018spider}. Since our main focus is to fairly compare different variance reduction methods, we do not try to achieve state-of-the-art accuracies using tricks like momentum, weight decay, learning rate reduction, etc. We defer further implementation details to the appendix.

\begin{figure*}[h!]
     \centering
     \begin{subfigure}[b]{0.32\textwidth}
         \centering
         \includegraphics[width=\textwidth]{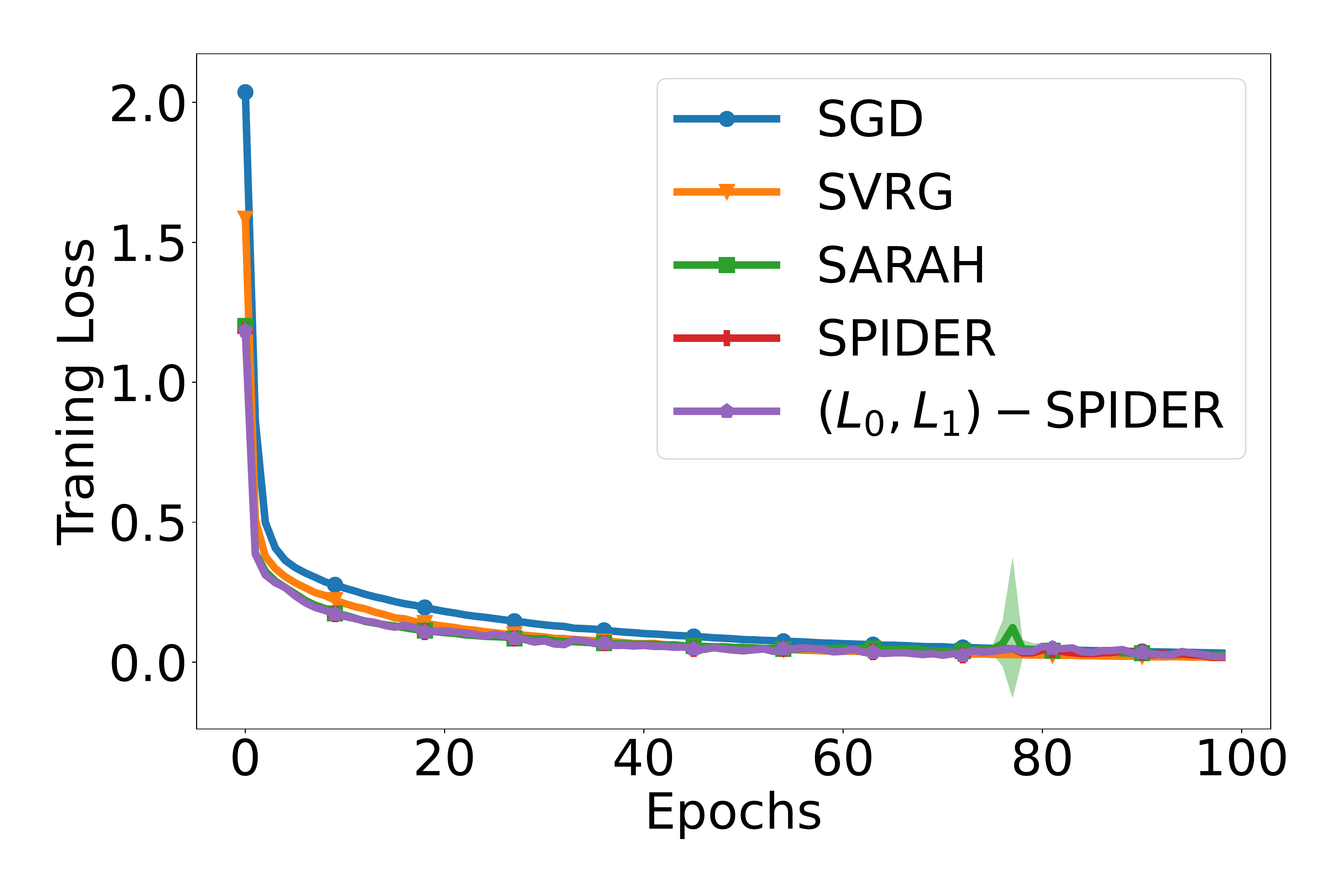}
         \includegraphics[width=\textwidth]{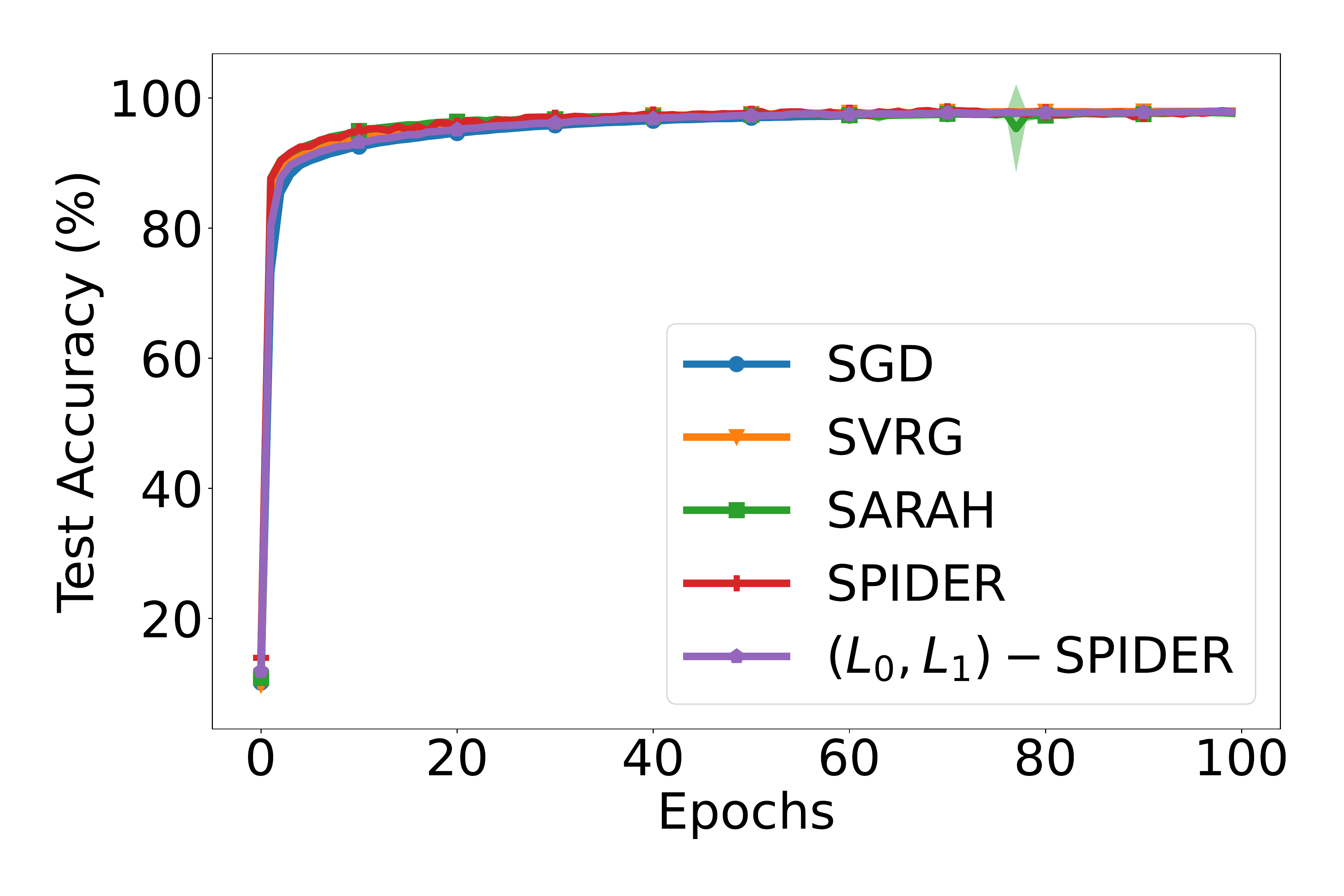}
         \caption{FCN on MNIST}
         \label{fig:fcn_mnist}
     \end{subfigure}
     \hfill
     \begin{subfigure}[b]{0.32\textwidth}
         \centering
         \includegraphics[width=\textwidth]{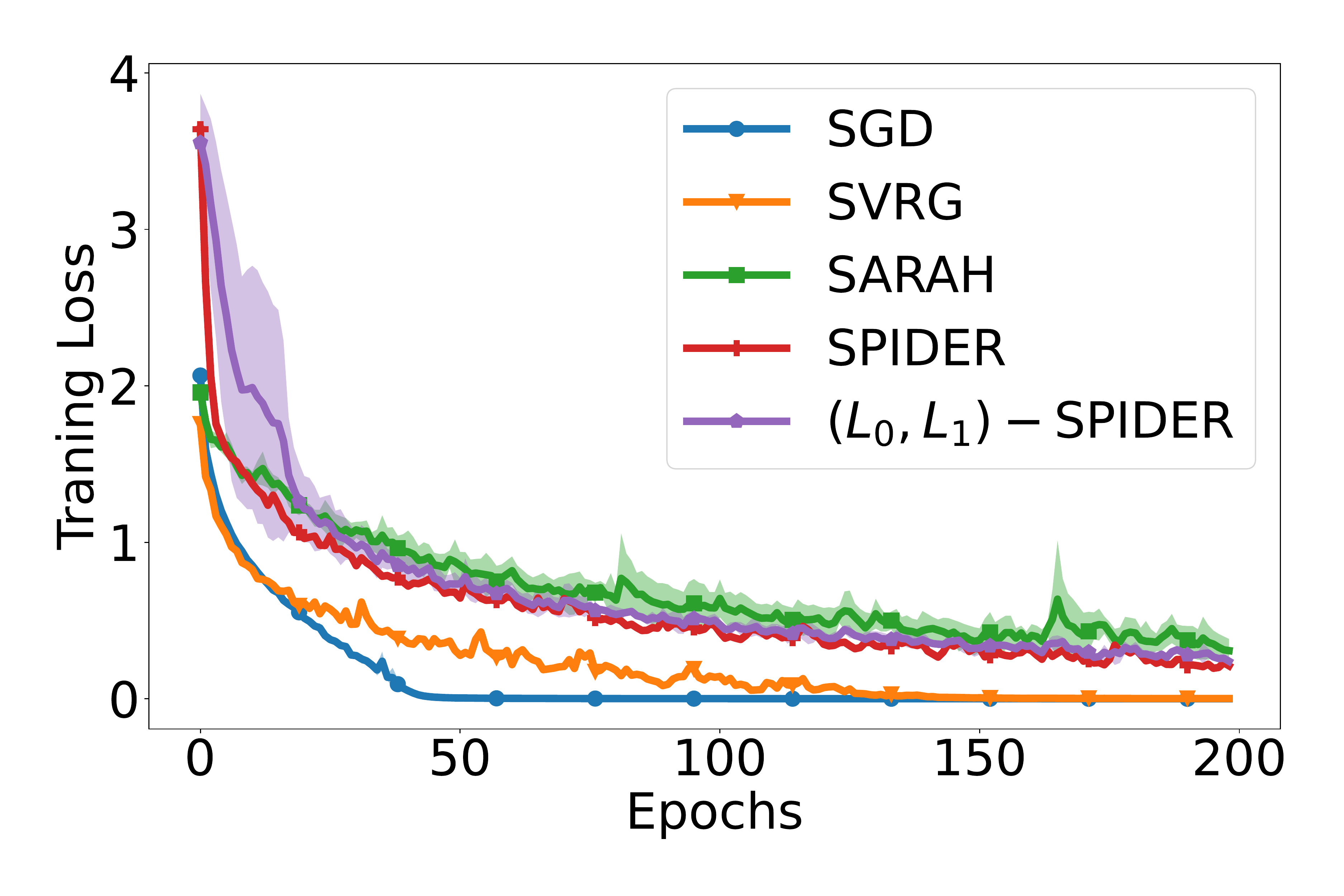}

         \includegraphics[width=\textwidth]{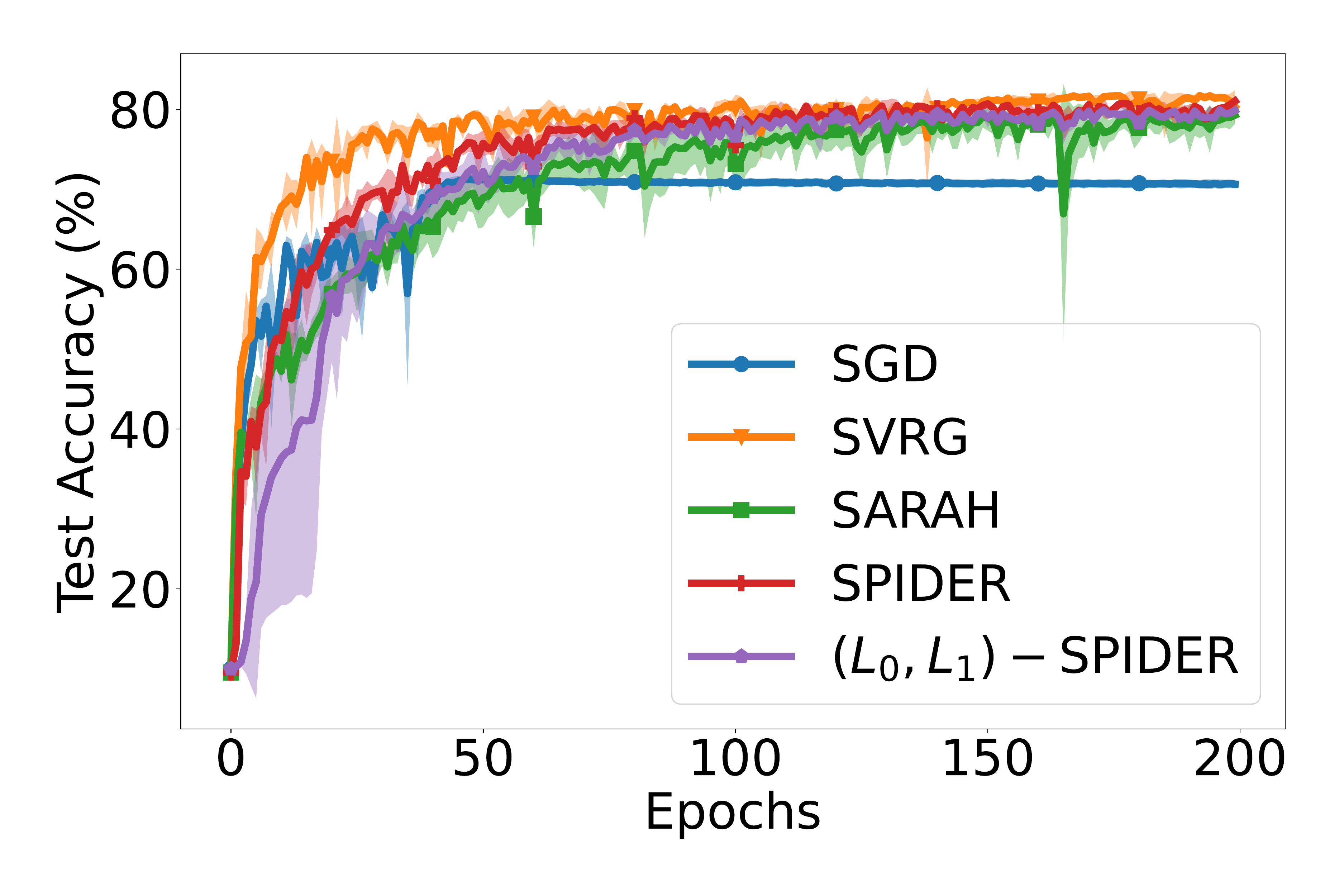}

         \caption{ResNet-20 on CIFAR10}
         \label{fig:resnet20_cifar10}
     \end{subfigure}
     \hfill
     \begin{subfigure}[b]{0.32\textwidth}
         \centering
         \includegraphics[width=\textwidth]{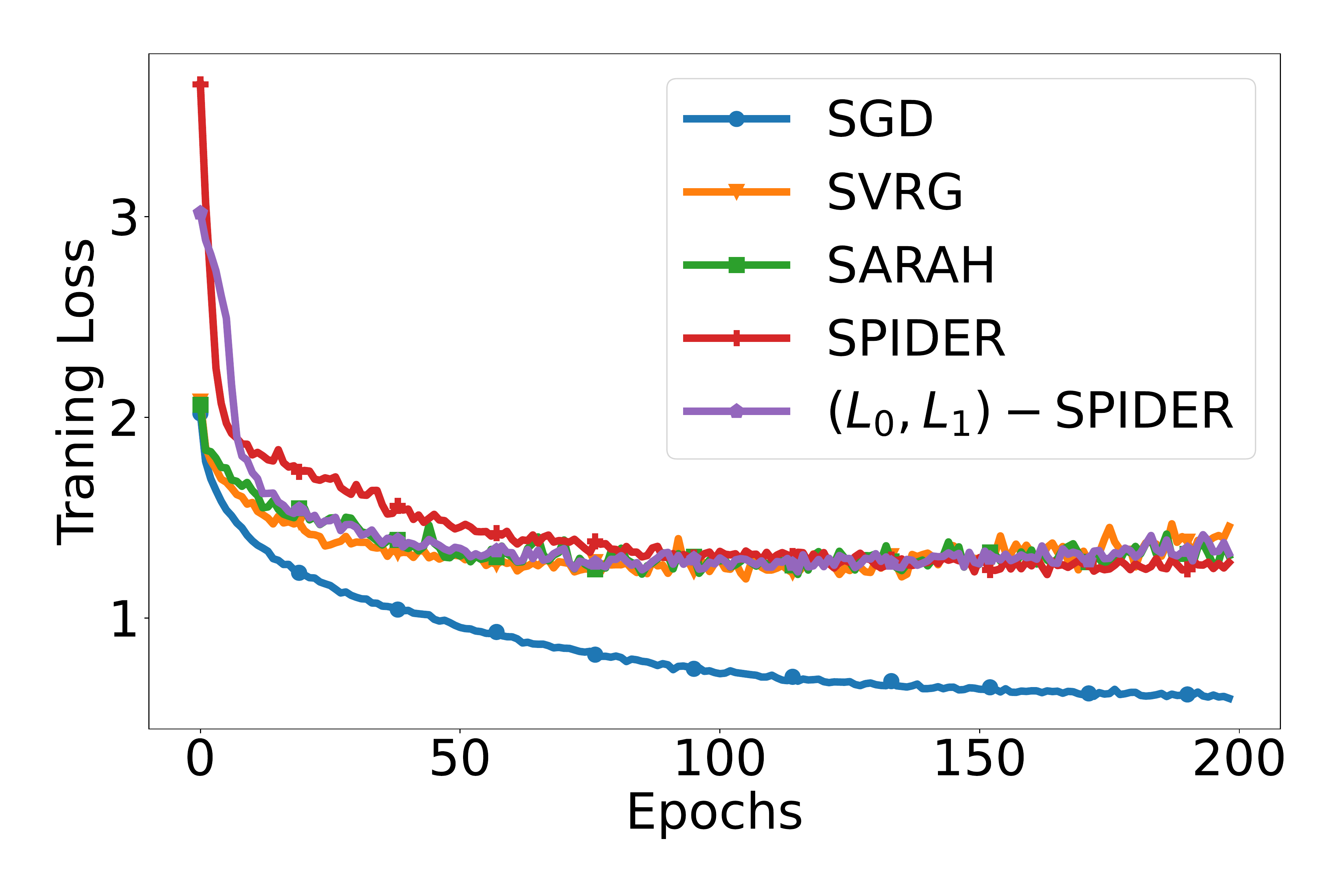}
         
         \includegraphics[width=\textwidth]{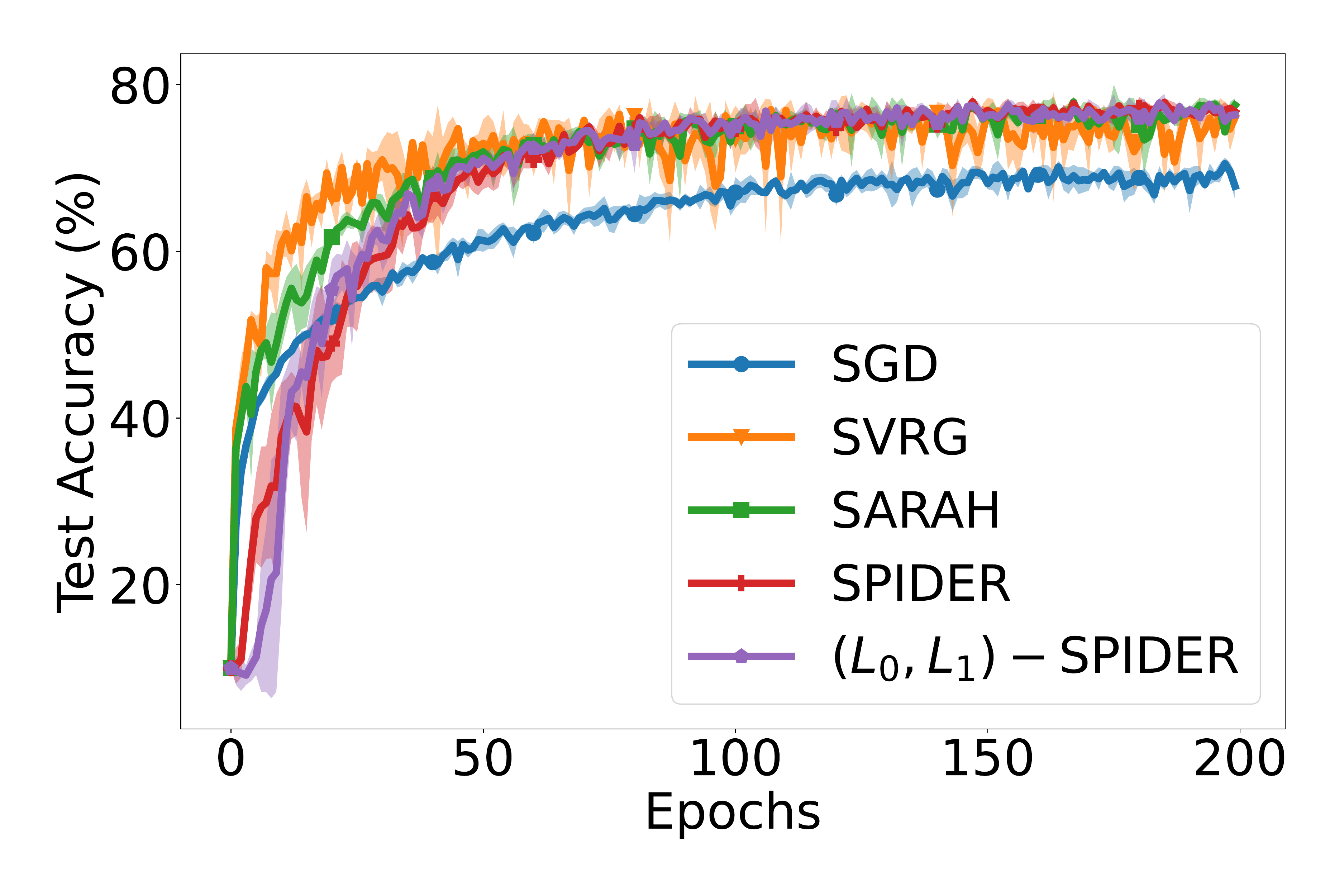}

         \caption{ResNet-20 on noisy CIFAR10}
         \label{fig:resnet20_cifar10_noisy}
     \end{subfigure}
        \\
     \begin{subfigure}[b]{0.32\textwidth}
         \centering
         \includegraphics[width=\textwidth]{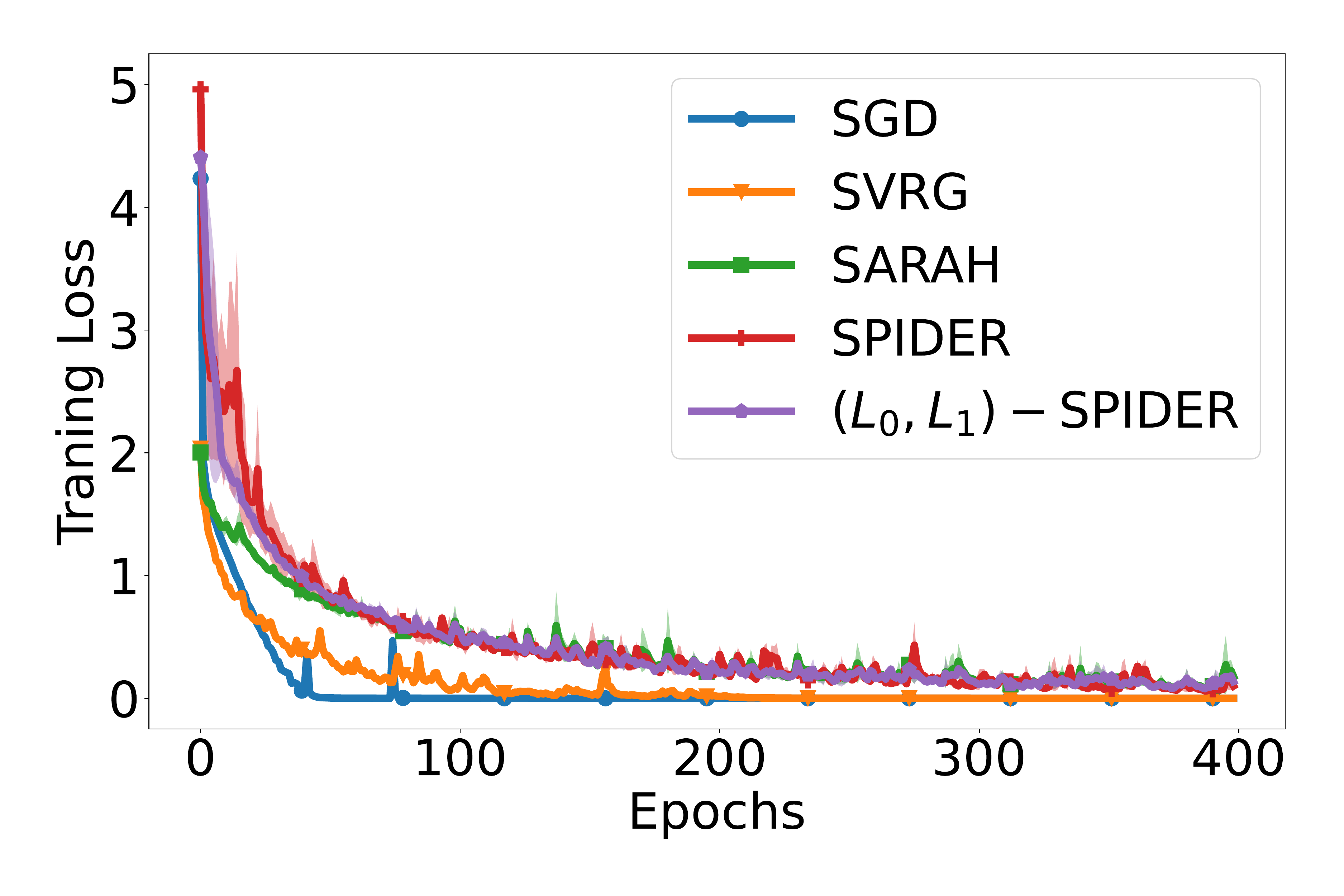}
         \includegraphics[width=\textwidth]{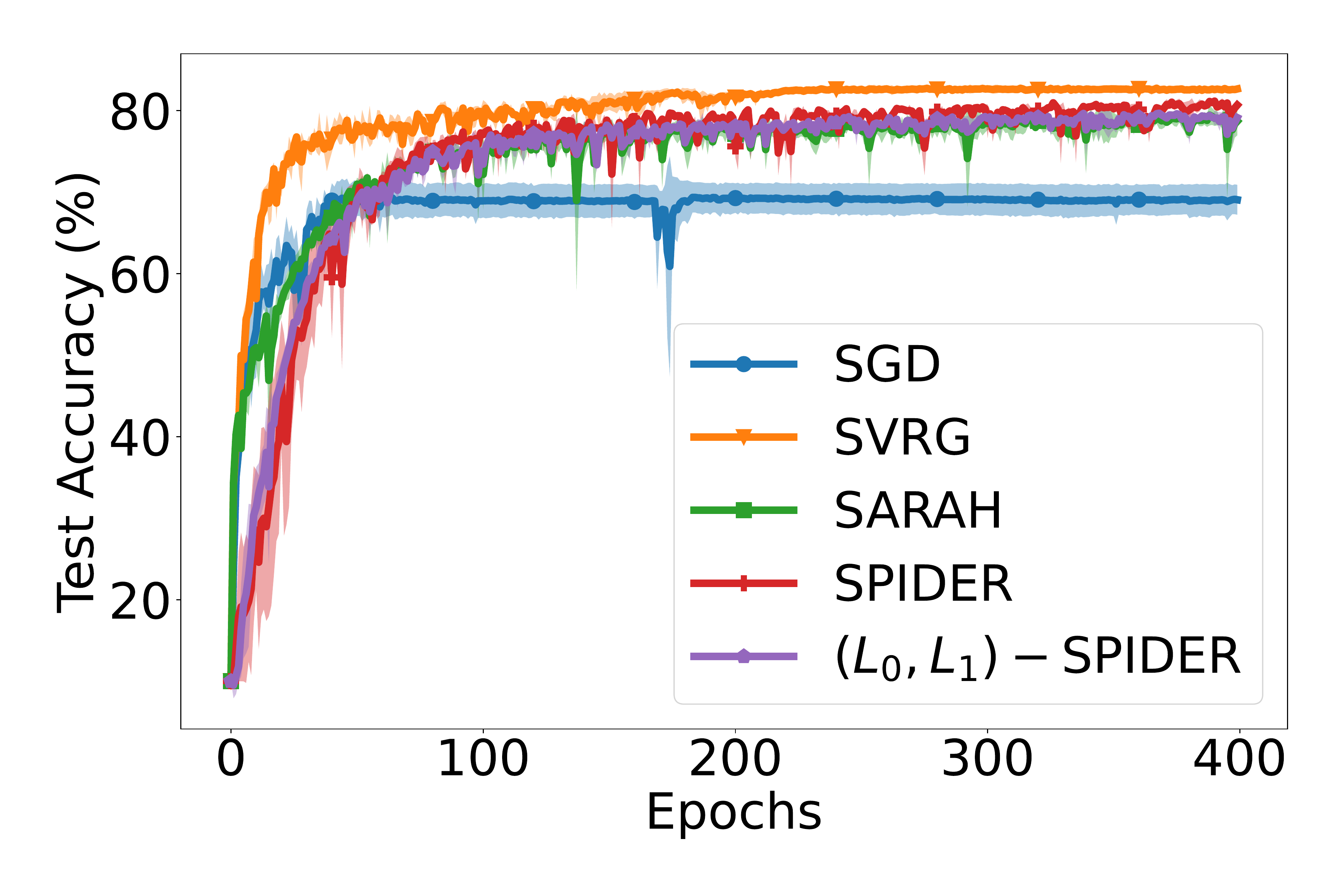}

         \caption{ResNet-56 on CIFAR10}
         \label{fig:resnet56_cifar10}
     \end{subfigure}
     \hfill
     \begin{subfigure}[b]{0.32\textwidth}
         \centering
         \includegraphics[width=\textwidth]{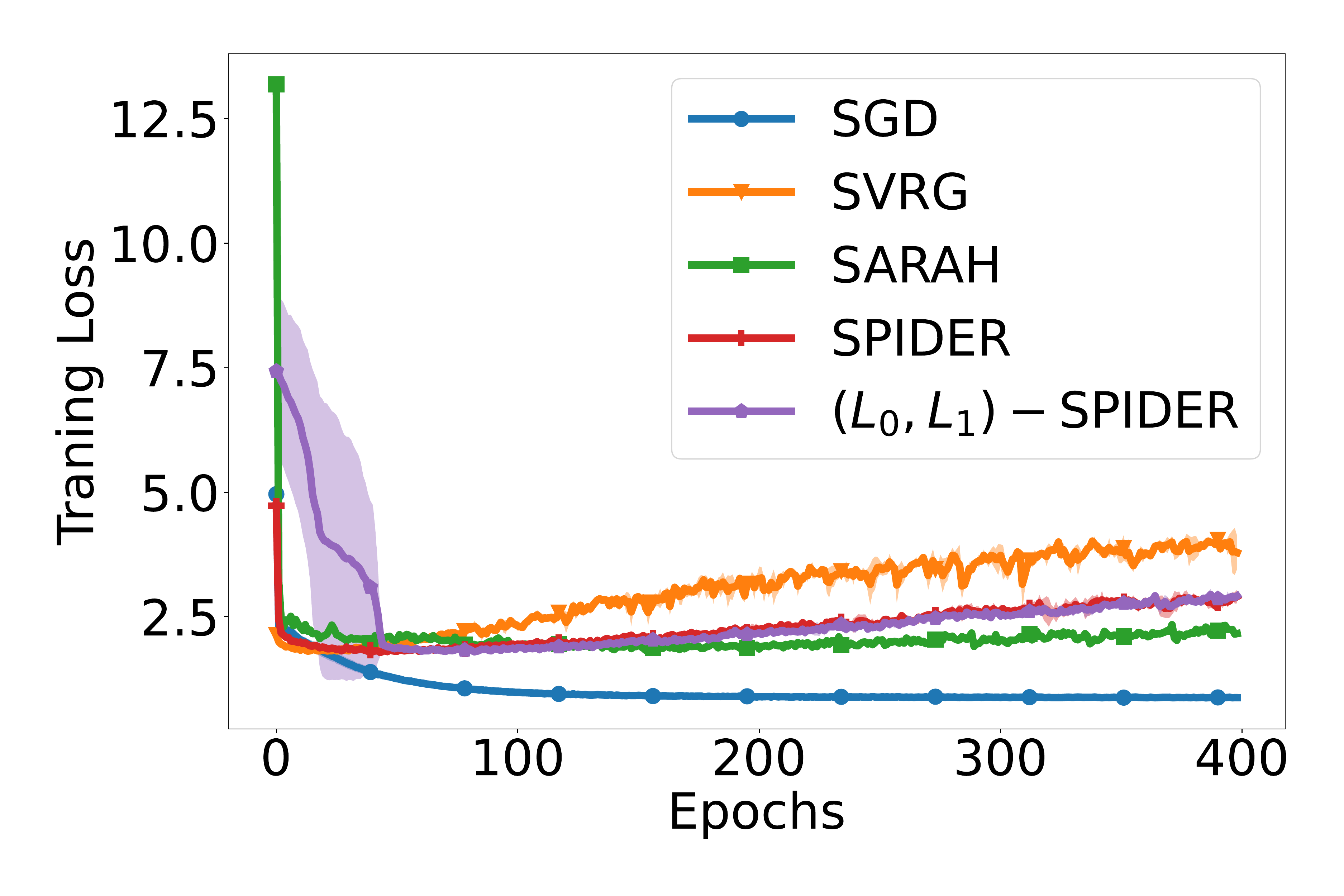}
         
         \includegraphics[width=\textwidth]{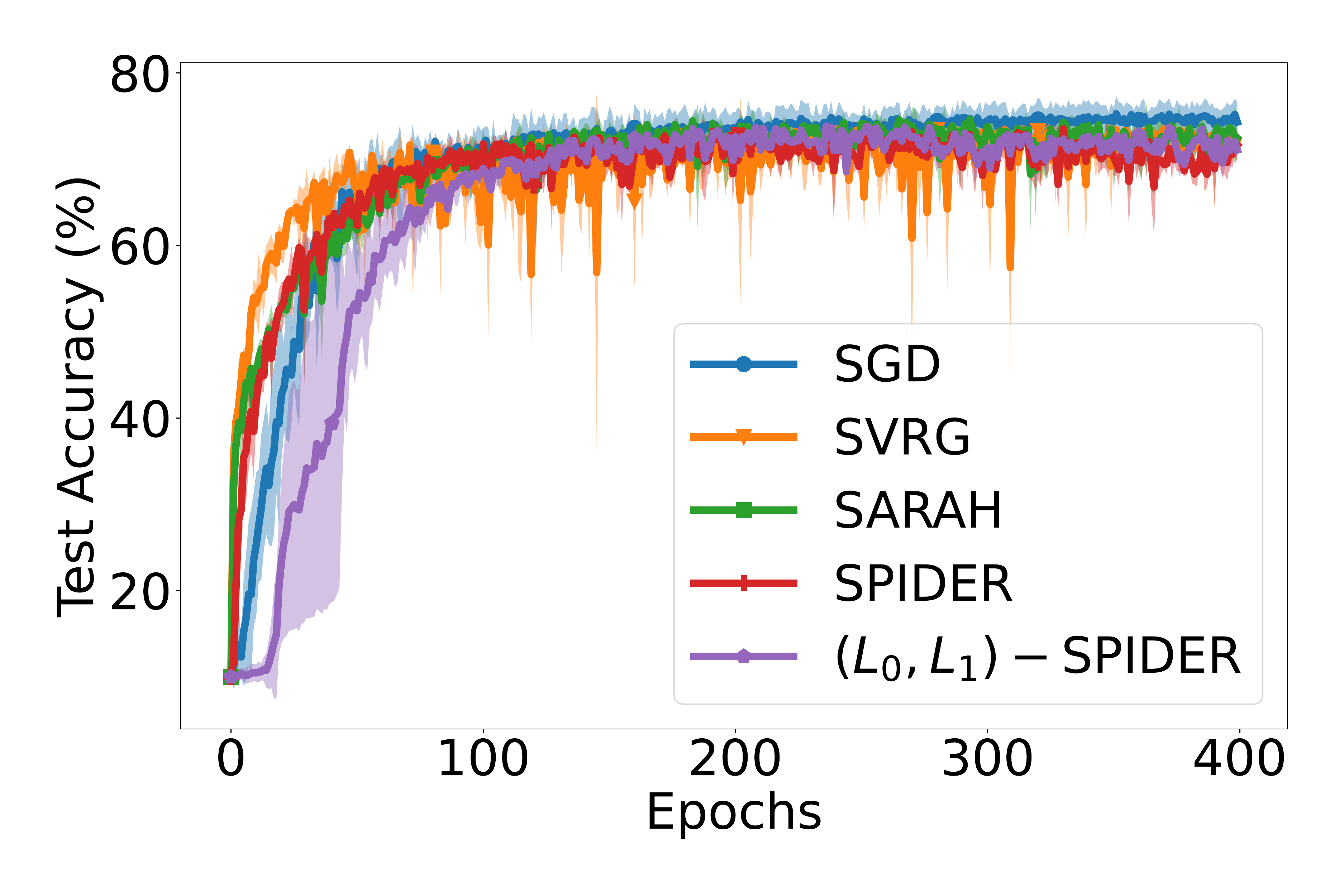}

         \caption{ResNet-56 on noisy CIFAR10}
         \label{fig:resnet56_cifar10_noisy}
     \end{subfigure}
     \hfill
          \begin{subfigure}[b]{0.32\textwidth}
         \centering
         \includegraphics[width=\textwidth]{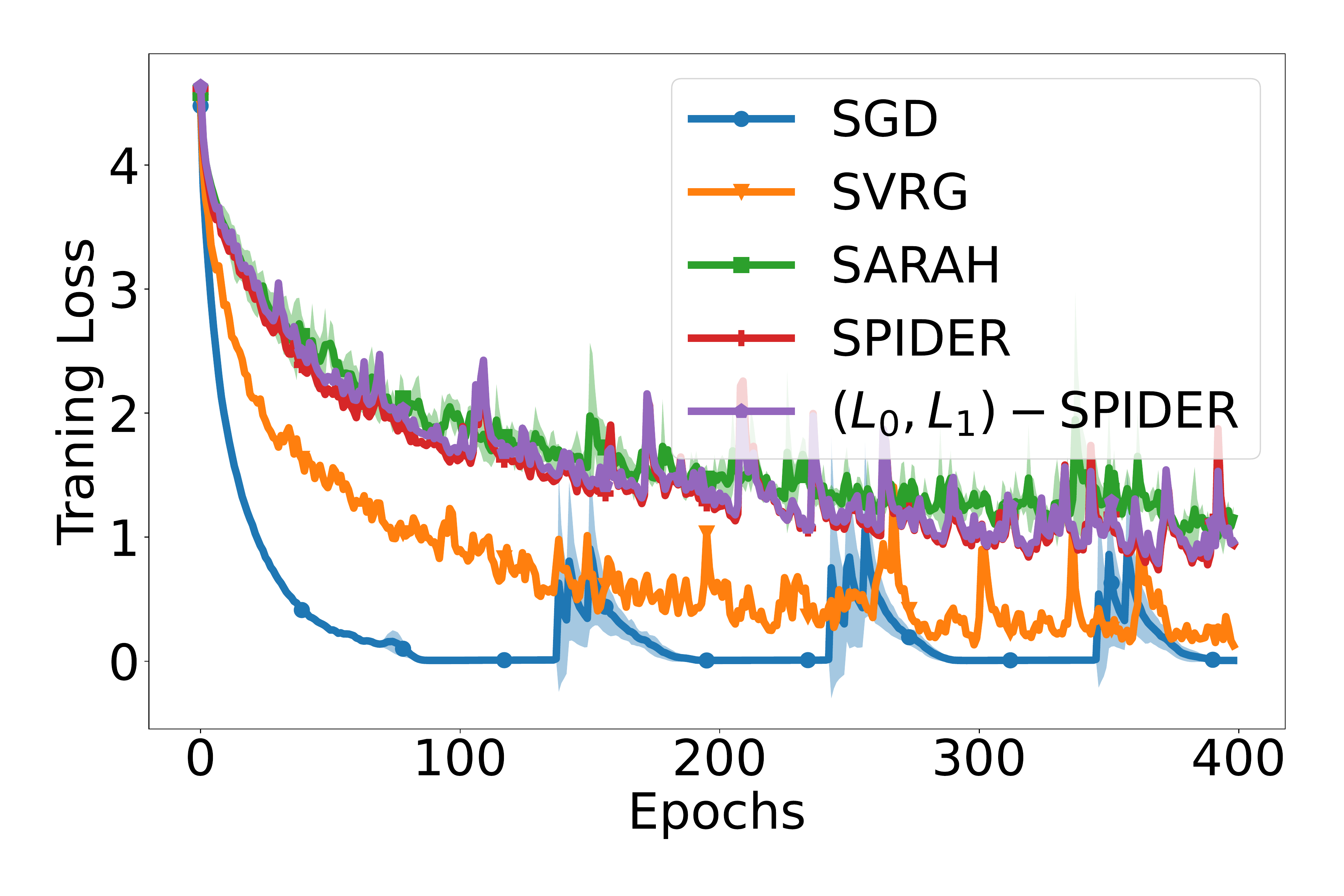}
         
         \includegraphics[width=\textwidth]{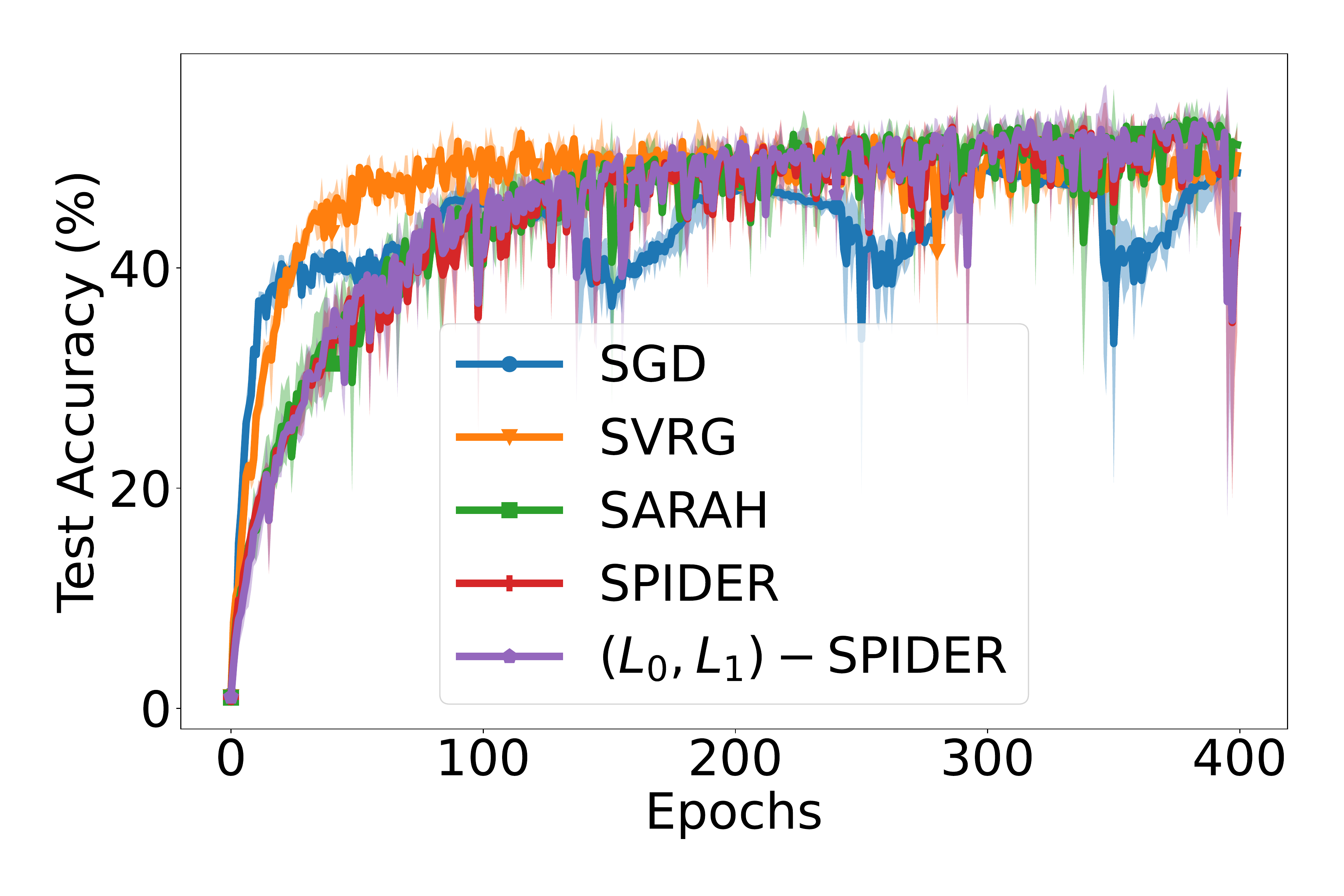}

         \caption{ResNet-20 on CIFAR100}
         \label{fig:resnet20_cifar100}
     \end{subfigure}
        \caption{Image classification tasks with neural networks.}
        \label{fig: all plots}
\end{figure*}

As demonstrated in Figure \ref{fig: all plots}(\subref{fig:fcn_mnist}), for the simple task of training FCN on MNIST, all methods achieve fairly high test accuracy and show almost identical performances. Figure \ref{fig: all plots}(\subref{fig:resnet20_cifar10}) provides training and test accuracy curves for ResNet-20 trained on clean (noiseless) CIFAR10. Although the test accuracy of {\SGD} is lower than that of variance reduction methods, it converges much faster during the training. This is in accordance with the conclusions in \citep{defazio2019ineffectiveness} that variance reduction methods are not very efficient for deep learning tasks. We highlight that our proposed $(L_0,L_1)$--{\Spider} achieves similar performance to the original {\Spider}, and the additional $\mathcal{O}(\|\bbv_k\|^{-2})$ term in the stepsize does not slow down the training much except in the initial steps. Quantitatively, $(L_0,L_1)$--{\Spider} attains $79.94\%$ test accuracy which is comparable to other variance-reduced benchmarks with $81.72\%$, $79.28\%$ and $81.03\%$ test accuracy. We provide further quantitative results in the appendix.



We further train the same model on noisy CIFAR10. That is, we add Gaussian noise with zero-mean and unit-variance to the images and change the label with probability $0.1$ to all possible categories uniformly. As Figure \ref{fig: all plots}(\subref{fig:resnet20_cifar10_noisy}) shows, although {\SVRG} is still slightly faster than other variance-reduced methods, both {\Spider} and $(L_0,L_1)$--{\Spider} achieve better test accuracy compared to the noiseless experiments in Figure \ref{fig: all plots}(\subref{fig:resnet20_cifar10}) which underscore the potential robustness of {\Spider} to noise.


Next, we train a larger model, that is ResNet-56 on both noiseless and noisy CIFAR10 dataset. As depicted in Figures \ref{fig: all plots}(\subref{fig:resnet56_cifar10}) and \ref{fig: all plots}(\subref{fig:resnet56_cifar10_noisy}), $(L_0,L_1)$--{\Spider} achieves similar performance to the other variance-reduced benchmarks. They also share other aspects in their convergence with those of the smaller model ResNet-20 in Figures \ref{fig: all plots}(\subref{fig:resnet20_cifar10}) and \ref{fig: all plots}(\subref{fig:resnet20_cifar10_noisy}). Finally, we use a more complicated task  with CIFAR100 dataset and train ResNet-20 with results demonstrated in Figure \ref{fig: all plots}(\subref{fig:resnet20_cifar100}).

\section{Conclusion}
Gradient clipping has been extensively used in training deep neural networks for particular applications such as language models. The training trajectory of gradient clipping for such models contradicts the traditional $L$--smoothness assumption which calls for relaxing this premise in non-convex optimization. The more relaxed $(L_0,L_1)$--smooth notion has laid out a theoretical framework to study the performance and complexity of gradient clipping methods. In this work, we improved the gradient complexity of clipping methods under this broader setting by employing a variance reduction technique called {\Spider}. We showed that {\Spider} with a carefully picked learning rate is able to achieve the order-optimal gradient complexity rates in finding first-order stationary points. It however remains to study how this method can be boosted to escape from saddle points and find  \emph{second-order} stationary solutions for $(L_0,L_1)$--smooth non-convex objectives. Similar to the first-order literature, most of the existing works on the second-order methods highly utilize the restrictive Hessian Lipschitz assumption which most likely breaks in the $(L_0,L_1)$--smooth setting. We leave this direction for future work.

\section*{Acknowledgments}
This work was supported, in part, by the MIT-IBM Watson
AI Lab and ONR Grant N00014-20-1-2394.

\bibliographystyle{plainnat}
\bibliography{ref}

\clearpage
\newpage

\section*{Appendices}
\appendix
\section{Experiment Setup} 

Here, we provide further details about our experiments along with quantitative illustrations of the results laid out in Figure \ref{fig: all plots} from the main paper.

As discussed in the main paper, we train three different neural networks, which are a three-layer fully connected network (FCN), ResNet-20 and ResNet-56 on several datasets, that are MNIST, CIFAR10 and CIFAR100. In each experiment, we compare the performance of {$(L_0,L_1)$--\Spider} against benchmarks including {\SGD}, {\SVRG}, {\SARAH} and {\Spider}.  For each method, its hyper-parameters like learning rate are tuned over a range to achieve the best possible test accuracy. In addition to the datasets mentioned above, we conduct experiments on their noisy versions. In particular, we train ResNet-20 on a noisy CIFAR10 dataset. Here, we add Gaussian noise to the images from CIFAR10 dataset with variance $1$ in $\ell_2$ norm and also change the label with probability $0.1$ to all possible categories uniformly. Moreover, we double the noise scale on CIFAR10 and  train ResNet-56 model. Figure \ref{fig: all plots} in the main paper demonstrates our results for the experiments described above where for each setting, we conduct three runs and show both mean and standard deviation. Furthermore, we provide quantitative implications from the same experiments in the following table. To obtain each entry in this table, we pick the best test accuracy along each of the three trajectories and report their average, as shown in Table~\ref{tab:my_label}. In the supplementary materials, we provide the code of our experiments which is modified from that of \citep{horvath2020adaptivity}.

\begin{table}[H]
\centering
\caption{Test accuracy ($\%$) corresponding to experiments in Figure \ref{fig: all plots}. $^*$Noisy data.
}
\begin{tabular}{lccccccc}\toprule
 & FCN & \multicolumn{3}{c}{ResNet-20} & \multicolumn{2}{c}{ResNet-56} 
\\\cmidrule(lr){3-5}\cmidrule(lr){6-7}
       Method   & MNIST &  CIFAR10  &  CIFAR10$^*$  & CIFAR100   &  CIFAR10  &  CIFAR10$^*$ \\\midrule
{\SGD}    & $97.83$ & $71.31$ & $70.52$ & $49.02$ & $69.36$ & $75.26$ \\
{\SVRG} & $97.98$ & $81.72$ & $77.00$ & $52.24$  &  $82.76$ & $73.49$ \\
{\SARAH} & $97.75$ & $79.28$ & $77.94$ & $53.45$ & $79.45$& $74.62$\\
{\Spider} & $98.04$ & $81.03$ &$77.95$  & $53.33$ & $81.2$ & $73.53$\\
{$(L_0,L_1)$--\Spider}   & $97.91$ & $79.94$ & $77.82$ & $53.39$ & $79.69$ &  $73.83$\\\bottomrule
\end{tabular}
\label{tab:my_label}
\end{table}

Next, we provide the detailed hyper-parameters corresponding to all the curves in Figure~\ref{fig: all plots}. First, the mini-batch size for all the methods and datasets is fixed as $1024$. For \SGD, \SVRG, and \SARAH, the only hyper-parameter to tune is the learning rate $\eta_0$. However, for \Spider, the stepsize at iteration $k$ is $\eta_k=\eta_0\min\{1, c_1/ \| \bbv_k \|\}$ which is determined by the learning rate $\eta_0$ and the clipping parameter $c_1$. For $(L_0,L_1)$--\Spider, the stepsize is $\eta_k=\eta_0\min\{1,c_1/\| \bbv_k \|,c_2/\| \bbv_k \|^2\}$ which is governed by the learning rate $\eta_0$ and two clipping parameters $c_1$ and $c_2$. All of the hyper-parameters are tuned to obtain the best test accuracy. We show the tuned learning rates for all experiments in Table~\ref{tab:lr}. Moreover, Table \ref{tab:clip} provides the tuned clipping parameters $c_1$ for {\Spider} and $(c_1,c_2)$ for $(L_0,L_1)$--{\Spider} methods. Note that we do not use techniques such as momentum or weight decay for any of our experiments except in training ResNet-20 on CIFAR100, for which we use a momentum parameter of $0.9$ and weight decay parameter of $10^{-4}$ to get a better test accuracy.

\begin{table}[H]
\centering
\begin{tabular}{lccccccc}\toprule
 & FCN & \multicolumn{3}{c}{ResNet-20} & \multicolumn{2}{c}{ResNet-56} 
\\\cmidrule(lr){3-5}\cmidrule(lr){6-7}
       Method   & MNIST &  CIFAR10  &  CIFAR10$^*$  & CIFAR100   &  CIFAR10  &  CIFAR10$^*$ \\\midrule
{\SGD}    & $0.1$ & $0.2$ & $0.025$ & $0.0125$ & $0.2$ & $1.6$ \\
{\SVRG} & $0.05$ & $0.2$ & $0.05$ & $0.8$  &  $0.2$ & $0.025$ \\
{\SARAH} & $0.025$ & $0.05$ & $0.025$ & $0.025$ & $0.0125$& $0.0125$\\
{\Spider} & $0.0125$ & $0.05$ &$0.05$  & $0.0125$ & $0.025$ & $0.0125$\\
{$(L_0,L_1)$--\Spider}   & $0.0125$ & $0.025$ & $0.025$ & $0.05$ & $0.0125$ &  $0.00625$\\\bottomrule
\end{tabular}
\caption{Learning rates corresponding to experiments in Figure \ref{fig: all plots}. $^*$Noisy data.
}
\label{tab:lr}
\end{table}

\begin{table}[H]
\centering
\begin{tabular}{lccccccc}\toprule
 & FCN & \multicolumn{3}{c}{ResNet-20} & \multicolumn{2}{c}{ResNet-56} 
\\\cmidrule(lr){3-5}\cmidrule(lr){6-7}
       Method   & MNIST &  CIFAR10  &  CIFAR10$^*$  & CIFAR100   &  CIFAR10  &  CIFAR10$^*$ \\\midrule
{\Spider} & $0.5$ & $1$ &$0.5$  & $16$ & $8$ & $16$\\
{$(L_0,L_1)$--\Spider}   & $(0.5,0.5)$ & $(2,2)$ & $(2,2)$ & $(16,128)$ & $(8,128)$ &  $(16,128)$\\\bottomrule
\end{tabular}
\caption{Clipping parameters corresponding to experiments in Figure \ref{fig: all plots}. $^*$Noisy data.
}
\label{tab:clip}
\end{table}

\section{Proof of Theorem \ref{thm: stochastic spider}} \label{sec: proof thm stochastic spider}

We first state and prove an essential lemma, a.k.a. the descent lemma, which is the common step in showing the convergence of non-convex optimization algorithms. Throughout this section, we use the notation $\E_k[\cdot]$ as the expectation conditioned on the history $\ccalF_k$ containing $\{\bbx_{0:k}, \bbv_{0:k-1} \}$.

\begin{lemma}[Descent Lemma] \label{lemma: descent}
Assume that $F$ is $(L_0,L_1)$--smooth according to Definition \ref{def: L0-L1} and consider the update $\bbx_{k+1} = \bbx_k - \eta_k \bbv_k$. Then, for $\epsilon \leq \frac{L_0}{2L_1}$ and stepsize 
\begin{align}
    \eta_k
    \leq
    \min\left\{ \frac{1}{2 L_0}, \frac{\epsilon}{L_0 \norm{\bbv_k}}\right\},
\end{align}
we have that for any iteration $k=0,1,\cdots$,
\begin{align}
    F(\bbx_{k+1})
    &\leq
    F(\bbx_k) - 
    \frac{1}{8} \eta_k \norm{\bbv_k}^2
    +
    \frac{5}{8} \eta_k \norm{\bbv_k - \gr F(\bbx_k)}^2.
\end{align}
\end{lemma}
\begin{proof}
We defer the proof to Section \ref{sec: proof descent lemma}.
\end{proof}

Another important step to show the convergence of variance reduction methods is to control the variance of the gradient estimator which is $\E \Vert \bbv_k - \gr F(\bbx_k) \Vert^2$ in Algorithm \ref{alg:Spider}\footnote{We misuse the expression ``variance of the gradient estimator'' to refer to $\E \Vert \bbv_k - \gr F(\bbx_k) \Vert^2$ here since $\bbv_k$ is \emph{not} an unbiased estimator for $\gr F(\bbx_k)$, that is, $\E[\bbv_k | \ccalF_k] \neq \gr F(\bbx_k)$.
However, it holds that $\E[\bbv_k] = \E[\gr F(\bbx_k)]$.}. In the following lemma, we show that by scaling the stepsize inversely with both $\Vert \bbv_k \Vert$ and $\Vert \bbv_k \Vert^2$, we are able to control the variance $\E \Vert \bbv_k - \gr F(\bbx_k) \Vert^2$.

\begin{lemma} \label{lemma: epsilon2 bound}
Let Assumptions \ref{assumption: L0-L1} (i) and \ref{assumption: stch gr} hold and assume that $\epsilon \leq \frac{L_0}{2L_1}$. Then, for stepsize and parameters picked as follows
\begin{align}
    \eta_k
    \leq
    \min\left\{ \frac{1}{L_0} \frac{\epsilon}{\norm{\bbv_k}}, \frac{1}{L_1} \frac{\epsilon}{\norm{\bbv_k}^2} \right\},
\end{align}
\begin{align}
    S_1 = \frac{4 \sigma^2}{\epsilon^2},
    \quad
    S_2 = 48 \frac{L_0}{L_1} \frac{\sigma}{\epsilon},
    \quad
    q = 2 \frac{L_0}{L_1} \frac{\sigma}{\epsilon},
\end{align}
we have that
\begin{align}
    \E_{k_0} \Big[ \norm{\bbv_k - \gr F(\bbx_k)}^2 \Big] \leq \epsilon^2,
\end{align}
where $k_0 \leq k$ denotes the most recent iterate to $k$ for which $q$ divides $k_0$, that is, $k_0 = \lfloor k/q \rfloor \cdot q$.
\end{lemma}

\begin{proof}
We defer the proof to Section \ref{sec: proof epsilon2 bound lemma}.
\end{proof}

\emph{Proof of Theorem \ref{thm: stochastic spider}:} Having set up these two main helper lemmas, we move to prove Theorem \ref{thm: stochastic spider}. First, note that the specified choice of the stepsize and the accuracy condition $\epsilon < \frac{L_0}{20L_1}$ in Theorem \ref{thm: stochastic spider} satisfy the ones required by Lemma \ref{lemma: descent}. Therefore, using this lemma and the condition $\eta_k \leq 1/(2L_0)$ we have that
\begin{align} \label{eq: descent 2}
    F(\bbx_{k+1})
    &\leq
    F(\bbx_k) - 
    \frac{1}{8} \eta_k \norm{\bbv_k}^2
    +
    \frac{5}{8} \eta_k \norm{\bbv_k - \gr F(\bbx_k)}^2
    \leq
    F(\bbx_k) - 
    \frac{1}{8} \eta_k \norm{\bbv_k}^2
    +
    \frac{5}{16L_0} \norm{\bbv_k - \gr F(\bbx_k)}^2.
\end{align}
Moreover, for the specified choice of the stepsize $\eta_k$, we have that
\begin{align} \label{eq: eta_k v_k2}
    \eta_k \norm{\bbv_k}^2
    &=
    \min \left\{ \frac{\norm{\bbv_k}^2}{2L_0}, \frac{\epsilon \norm{\bbv_k}}{L_0}, \frac{\epsilon}{L_1} \right\}\\
    &=
    \min \left\{ \frac{\epsilon^2}{L_0} \min \left\{ \frac{1}{2}\norm{\frac{\bbv_k}{\epsilon}}^2, \norm{\frac{\bbv_k}{\epsilon}} \right\}, \frac{\epsilon}{L_1}\right\} \\
    &\overset{(a)}{\geq}
    \min \left\{ \frac{\epsilon}{L_0} \norm{\bbv_k} - \frac{2\epsilon^2}{L_0}, \frac{\epsilon}{L_1}\right\} \\
    &\geq
    \frac{\epsilon}{L_0} \min \left\{ \norm{\bbv_k}, \frac{L_0}{L_1}\right\} - \frac{2\epsilon^2}{L_0},
\end{align}
where in $(a)$, we used the inequality $\min\{x^2/2,|x|\} \geq |x| - 2$ for all $x$. Rearranging terms in \eqref{eq: descent 2} combined with \eqref{eq: eta_k v_k2} yields that
\begin{align}
    \frac{\epsilon}{8L_0} \min \left\{ \norm{\bbv_k}, \frac{L_0}{L_1}\right\} - \frac{\epsilon^2}{4L_0}
    \leq
    F(\bbx_k) - F(\bbx_{k+1})
    +
    \frac{5}{16L_0} \norm{\bbv_k - \gr F(\bbx_k)}^2.
\end{align}
Next, we  take expectations from both sides of the above inequality and use the bound in Lemma \ref{lemma: epsilon2 bound} which yields that
\begin{align}
    \frac{\epsilon}{8L_0} \E \left[ \min \left\{ \norm{\bbv_k}, \frac{L_0}{L_1}\right\} \right]
    \leq
    \E[F(\bbx_k)] - \E[F(\bbx_{k+1})]
    +
    \frac{9}{16L_0} \epsilon^2.
\end{align}
Now, we take the average of both sides over $k=0,\cdots,K-1$ which implies that
\begin{align}
    \frac{\epsilon}{8L_0} \cdot \frac{1}{K} \sum_{k=0}^{K-1} \E \left[ \min \left\{ \norm{\bbv_k}, \frac{L_0}{L_1}\right\} \right]
    \leq
    \frac{F(\bbx_0) - \E[F(\bbx_{K})]}{K}
    +
    \frac{9}{16L_0} \epsilon^2.
\end{align}
Multiplying both sides by $\frac{8L_0}{\epsilon}$ and using the fact that $F(\bbx_0) - \E[F(\bbx_{K})] \leq F(\bbx_0) - F^* = \Delta$ yields that
\begin{align} \label{eq: average E}
    \frac{1}{K} \sum_{k=0}^{K-1} \E \left[ \min \left\{ \norm{\bbv_k}, \frac{L_0}{L_1}\right\} \right]
    \leq
    \frac{8 \Delta L_0}{\epsilon K}
    +
    \frac{9}{2} \epsilon 
    \leq
    5 \epsilon,
\end{align}
where we employed the following choice of the number of iterations
\begin{align}
    K
    =
    \bigg\lceil \frac{16 \Delta L_0}{\epsilon^2} \bigg\rceil.
\end{align}
Now, consider index $\tilde{k}$ uniformly picked from $\{0,\cdots,K-1\}$ at random. The average argument in \eqref{eq: average E} implies that
\begin{align}
    \E \left[ \min \left\{ \norm{\bbv_{\tilde{k}}}, \frac{L_0}{L_1}\right\} \right]
    \leq
    5 \epsilon,
\end{align}
where the expectation is w.r.t the randomness in both $\tilde{k}$ and the algorithm.
Next, we use Markov's inequality to yield that with probability at least $3/4$, we have
\begin{align} \label{eq: markov v}
    \min \left\{ \norm{\bbv_{\tilde{k}}}, \frac{L_0}{L_1}\right\}
    \leq
    20 \epsilon.
\end{align}
Note that for $\epsilon < \frac{L_0}{20 L_1}$, we have $L_0/L_1 < 20 \epsilon$. Therefore, the above bound simplifies to $\norm{\bbv_{\tilde{k}}} \leq 20 \epsilon$.

Next, we have from Lemma \ref{lemma: epsilon2 bound} that $\E [\Vert \bbv_k - \gr F(\bbx_k)\Vert^2 ] \leq \epsilon^2$ for every $k$.
For uniformly picked $\tilde{k} \in \{0,\cdots,K-1\}$ we have
\begin{align}
    \E \left[\Vert \bbv_{\tilde{k}} - \gr F(\bbx_{\tilde{k}})\Vert^2 \right] =
    \frac{1}{K} \sum_{k=0}^{K-1} \E \left[ \big\Vert \bbv_k - \gr F(\bbx_k)\big\Vert^2 \right]
    \leq
    \epsilon^2,
\end{align}
which together with Jensen's inequality implies that $\E \Vert \bbv_{\tilde{k}} - \gr F(\bbx_{\tilde{k}})\Vert \leq \epsilon$. 
Using Markov's inequality, we have $\Vert \bbv_{\tilde{k}} - \gr F(\bbx_{\tilde{k}})\Vert \leq 4 \epsilon$ with probability $3/4$. Finally, we use union bound on the above two good events to conclude that for randomly and uniformly picked index $\tilde{k} \in \{0,\cdots,K-1\}$,
\begin{align} 
    \norm{\gr F(\bbx_{\tilde{k}})}
    \leq
    \norm{\bbv_{\tilde{k}}}
    +
    \norm{\bbv_{\tilde{k}} - \gr F(\bbx_{\tilde{k}})}
    \leq
    20 \epsilon + 4 \epsilon
    =
    24 \epsilon,
\end{align}
with probability at least $1 - 1/4 - 1/4 = 1/2$.

\textbf{Total iteration complexity:} The total gradient complexity of {\Spider} in Algorithm \ref{alg:Spider} can be bounded as follows
\begin{align} 
    \bigg\lceil K \cdot \frac{1}{q} \bigg\rceil S_1 + K S_2
    &\leq
    K \cdot \frac{1}{q} \cdot S_1 + S_1 + K S_2 \\
    &\leq
    \left( \frac{16 \Delta L_0}{\epsilon^2} + 1 \right) \frac{L_1 \epsilon}{2 L_0 \sigma} \cdot \frac{4 \sigma^2}{\epsilon^2}
    +
    \frac{4 \sigma^2}{\epsilon^2}
    +
    \left( \frac{16 \Delta L_0}{\epsilon^2} + 1 \right) \frac{48 L_0 \sigma}{L_1 \epsilon} \\
    &=
    32 \Delta \sigma \left(L_1 + 24 \frac{L_0^2}{L_1} \right) \frac{1}{\epsilon^3}
    +
    \frac{4 \sigma^2}{\epsilon^2}
    +
    2 \sigma \left( \frac{L_1}{L_0} + 24 \frac{L_0}{L_1}\right) \frac{1}{\epsilon}.
\end{align}

\section{Proof of Theorem \ref{thm: finite-sum spider}} \label{sec: proof thm finite-sum spider}

Before proving Theorem \ref{thm: finite-sum spider} which corresponds to the finite-sum setting, we provide two helper lemmas. First, Lemma \ref{lemma: descent} can be directly employed in the finite-sum setting, as well as the stochastic setting. Second, we bound the variance $\E \Vert \bbv_k - \gr F(\bbx_k) \Vert^2$ in the finite-sum setting, similar to Lemma \ref{lemma: epsilon2 bound} in the stochastic setting.

\begin{lemma} \label{lemma: epsilon2 bound finite-sum}
Let Assumption \ref{assumption: L0-L1} (ii) hold and assume that $\epsilon \leq \frac{L_0}{2L_1}$. Then, for stepsize and parameters picked as follows
\begin{align}
    \eta_k
    \leq
    \min\left\{ \frac{1}{L_0} \frac{\epsilon}{\norm{\bbv_k}}, \frac{1}{L_1} \frac{\epsilon}{\norm{\bbv_k}^2} \right\}
\end{align}
\begin{align}
    S_1 = n,
    \quad
    S_2 = 12 \sqrt{n},
    \quad
    q = \sqrt{n},
\end{align}
we have that
\begin{align}
    \E_{k_0} \Big[ \norm{\bbv_k - \gr F(\bbx_k)}^2 \Big] \leq \epsilon^2,
\end{align}
where $k_0 \leq k$ denotes the most recent iterate to $k$ for which $q$ divides $k_0$, that is, $k_0 = \lfloor k/q \rfloor \cdot q$.
\end{lemma}

\begin{proof}
We defer the proof to Section \ref{sec: proof lemma epsilon2 bound finite-sum}.
\end{proof}

\emph{Proof of Theorem \ref{thm: finite-sum spider}:} Using the descent lemma in Lemma \ref{lemma: descent} and the variance bound established in Lemma \ref{lemma: epsilon2 bound finite-sum}, the rest of the proof follows from the proof of Theorem \ref{thm: stochastic spider}. Particularly, for randomly and uniformly picked index ${\tilde{k}} \in \{0,\cdots,K-1\}$, we have $\Vert \gr F(\bbx_{\tilde{k}})\Vert \leq 24 \epsilon$ with probability at least $1/2$. The total iteration complexity of finding such a stationary point is bounded by
\begin{align} 
    \bigg\lceil K \cdot \frac{1}{q} \bigg\rceil S_1 + K S_2
    &\leq
    K \cdot \frac{1}{q} \cdot S_1 + S_1 + K S_2 \\
    &\leq
    \left( \frac{16 \Delta L_0}{\epsilon^2} + 1 \right) \frac{1}{\sqrt{n}} \cdot n
    +
    n
    +
    \left( \frac{16 \Delta L_0}{\epsilon^2} + 1 \right) 12 \sqrt{n} \\
    &=
    208 \Delta L_0 \sqrt{n} \frac{1}{\epsilon^2}
    +
    n
    +
    13 \sqrt{n}.
\end{align}

\section{Proof of Theorem \ref{thm: stochastic ClippedSGD}} \label{sec: proof thm stochastic ClippedSGD}

We first employ the Descent Lemma (Lemma \ref{lemma: descent} and treating $\bbv_k$ as stochastic gradient $\bbg_k$) which implies that
\begin{align} \label{eq: descent Clipped SGD stochastic}
    F(\bbx_{k+1})
    \leq
    F(\bbx_k) - 
    \frac{1}{8} \eta_k \norm{\bbg_k}^2
    +
    \frac{5}{16L_0} \norm{\bbg_k - \gr F(\bbx_k)}^2.
\end{align}
Next, for the specified choice of stepsize $\eta_k$, we can write that
\begin{align} \label{eq: eta g2 bound}
    \eta_k \norm{\bbg_k}^2
    &=
    \min \left\{ \frac{\norm{\bbg_k}^2}{2L_0}, \frac{\epsilon \norm{\bbg_k}}{L_0}\right\}\\
    &=
    \frac{\epsilon^2}{L_0} \min \left\{ \frac{1}{2}\norm{\frac{\bbg_k}{\epsilon}}^2, \norm{\frac{\bbg_k}{\epsilon}} \right\} \\
    &\overset{(a)}{\geq}
    \frac{\epsilon}{L_0} \norm{\bbg_k} - \frac{2\epsilon^2}{L_0},
\end{align}
where $(a)$ follows from the fact that $\min\{x^2/2,|x|\} \geq |x| - 2$ for all $x$. Plugging this back into \eqref{eq: descent Clipped SGD stochastic} yields that
\begin{align}
    \frac{\epsilon}{8L_0} \norm{\bbg_k} - \frac{\epsilon^2}{4L_0}
    \leq
    F(\bbx_k) -  F(\bbx_{k+1})
    +
    \frac{5}{16L_0} \norm{\bbg_k - \gr F(\bbx_k)}^2.
\end{align}
We take expectation from both sides of the above which together with $\E \Vert \bbg_k - \gr F(\bbx_k) \Vert^2 \leq \frac{\sigma^2}{S}$ yields that
\begin{align}
    \E \norm{\bbg_k}
    \leq
    \frac{8L_0}{\epsilon} \left( \E[F(\bbx_k)] - \E[F(\bbx_{k+1})] \right)
    +
    \frac{5}{2\epsilon} \cdot  \frac{\sigma^2}{S}
    +
    2\epsilon.
\end{align}
Next, we sum the above inequality over $k=0,\cdots, K-1$ and conclude that
\begin{align} \label{eq: average g_k}
    \frac{1}{K} \sum_{k=0}^{K-1} \E \norm{\bbg_k}
    \leq
    \frac{8 \Delta L_0}{\epsilon K}
    +
    \frac{5}{2\epsilon} \cdot  \frac{\sigma^2}{S}
    +
    2\epsilon
    \leq
    \frac{\epsilon}{2} + \frac{5}{2}\epsilon + 2 \epsilon
    =
    5 \epsilon,
\end{align}
where we used the parameter choices 
\begin{align}
    S = \frac{\sigma^2}{\epsilon^2},
    \quad
    K = \bigg\lceil \frac{16 \Delta L_0}{\epsilon^2} \bigg\rceil.
\end{align}
Now, consider index $\tilde{k}$ uniformly picked from $\{0,\cdots,K-1\}$ at random. The average argument in \eqref{eq: average g_k} implies that $\E \Vert \bbg_{\tilde{k}} \Vert \leq 5 \epsilon$ where the expectation is w.r.t the randomness in both $\tilde{k}$ and the algorithm. Moreover, the variance of the stochastic gradient $\bbg_k$ is bounded by  $\E \Vert \bbg_k - \gr F(\bbx_k) \Vert^2 \leq \frac{\sigma^2}{S} = \epsilon^2$ for every $k$, which together with Jensen's inequality yields that  $\E \Vert \bbg_{\tilde{k}} - \gr F(\bbx_{\tilde{k}}) \Vert \leq \epsilon$. Therefore, we can write
\begin{align} 
    \E \norm{\gr F(\bbx_{\tilde{k}})}
    \leq
    \E \norm{\bbg_{\tilde{k}}}
    +
    \E \norm{\bbg_{\tilde{k}} - \gr F(\bbx_{\tilde{k}})}
    \leq
    6 \epsilon.
\end{align}
Finally, we use Markov's inequality to conclude that for randomly and uniformly picked index ${\tilde{k}} \in \{0,\cdots,K-1\}$, we have $\Vert \gr F(\bbx_{\tilde{k}}) \Vert \leq 12 \epsilon$  with probability at least $1/2$.

Total gradient complexity can be bounded as follows
\begin{align} 
    K S
    \leq
    \left( \frac{16 \Delta L_0}{\epsilon^2} + 1 \right) \frac{\sigma^2}{\epsilon^2}
    =
    16 \Delta L_0 \sigma^2 \frac{1}{\epsilon^4}
    +
    \frac{\sigma^2}{\epsilon^2}.
\end{align}

\section{Proof of Theorem \ref{thm: finite-sum ClippedSGD}} \label{sec: proof thm finite-sum ClippedSGD}
First, note that when using the full batch ($S=n$), we have $\bbg_k = \gr F(\bbx_k)$. Using the Descent Lemma (Lemma \ref{lemma: descent} and treating $\bbv_k = \gr F(\bbx_k)$), we have that
\begin{align} \label{eq: descent Clipped SGD stochastic}
    F(\bbx_{k+1})
    \leq
    F(\bbx_k) - 
    \frac{1}{8} \eta_k \norm{\gr F(\bbx_k)}^2.
\end{align}
Next, following the same argument as in \eqref{eq: eta g2 bound}, we have that
\begin{align}
    \frac{\epsilon}{8L_0} \norm{\gr F(\bbx_k)} - \frac{\epsilon^2}{4L_0}
    \leq
    F(\bbx_k) -  F(\bbx_{k+1}).
\end{align}
Summing over $k=0,\cdots,K-1$ yields that
\begin{align}
    \frac{1}{K} \sum_{k=0}^{K-1} \norm{\gr F(\bbx_k)}
    \leq
    \frac{8 \Delta L_0}{\epsilon K}
    +
    2\epsilon
    \leq
    \frac{\epsilon}{2} + 2 \epsilon
    =
    \frac{5}{2}\epsilon.
\end{align}
Let iterate ${\tilde{k}}$ be uniformly picked from $\{0,\cdots,K-1\}$ at random. The average argument above implies that $\E \Vert \gr F(\bbx_{\tilde{k}}) \Vert \leq \frac{5}{2} \epsilon$ which together with Markov's inequality implies that $\Vert \gr F(\bbx_{\tilde{k}}) \Vert \leq 5 \epsilon$ with probability at least $1/2$.

Total gradient complexity can be bounded as follows
\begin{align} 
    K S
    \leq
    \left( \frac{16 \Delta L_0}{\epsilon^2} + 1 \right) n
    =
    16 \Delta L_0 n \frac{1}{\epsilon^2}
    +
    n.
\end{align}

\section{Proof of Deferred Lemmas}

\subsection{Proof of Proposition \ref{prop: L0-L1}} \label{sec: proof prop L0-L1}
Let condition \eqref{eq: L0-L1 v2} hold. Then, for any unit-norm vector $\bbu$, we have that
\begin{align}
    \norm{\gr^2 F(\bbx) \bbu }
    =
    \Big\Vert\lim_{t \to 0} \frac{1}{t} \left(\gr F(\bbx + t \bbu) - \gr F(\bbx)\right)\Big\Vert
    =
    \lim_{t \to 0} \frac{1}{t} \norm{\gr F(\bbx + t \bbu) - \gr F(\bbx)}
    \leq
    L_0 + L_1 \Vert \gr F(\bbx) \Vert.
\end{align}
Therefore, we have $\Vert \gr^2 F(\bbx) \Vert = \sup_{\Vert \bbu \Vert = 1} \Vert \gr^2 F(\bbx) \bbu \Vert  \leq L_0 + L_1 \Vert \gr F(\bbx) \Vert$, which is the same as condition \eqref{eq: L0-L1 v1}. The other direction is a special case of the following result proved in \citep{zhang2020improved}.

\begin{lemma}[Corollary A.4 in \cite{zhang2020improved}]
Assume that $F$ satisfies \eqref{eq: L0-L1 v1}, that is, $\Vert \gr^2 F(\bbx) \Vert  \leq L_0 + L_1 \Vert \gr F(\bbx) \Vert$ for all $\bbx$. For any $c > 0$, if $\Vert \bbx - \bby \Vert \leq c/L_1$, then
\begin{align}
    \Vert \gr F(\bbx) - \gr F(\bby) \Vert
    \leq
    \big( A L_0 + B L_1 \Vert \gr F(\bbx) \Vert \big) \Vert \bbx - \bby \Vert,
\end{align}
where $A = 1 + e^c - \frac{e^c - 1}{c}$ and $B = \frac{e^c - 1}{c}$.
\end{lemma}
Taking $c=1$ in the above lemma yields that $A = 2$ and $B = e - 1 \leq 2$. Therefore, condition \eqref{eq: L0-L1 v2} holds with $2L_0$ and $2L_1$.

\subsection{Proof of Lemma \ref{lemma: descent}} \label{sec: proof descent lemma}

Consider any iteration $k$ and define $\bbx(t) = t(\bbx_{k+1} - \bbx_k) + \bbx_k$ for any $t \in [0,1]$ which lies between $\bbx_k$ and $\bbx_{k+1}$. From Taylor's Theorem, we have that
\begin{align} \label{eq: descent lemma 1}
    F(\bbx_{k+1})
    &=
    F(\bbx_k) + \int_0^1 \langle \gr F(\bbx(t)), \bbx_{k+1} - \bbx_k \rangle dt\\
    &=
    F(\bbx_k) +\langle \gr F(\bbx_k), \bbx_{k+1} - \bbx_k \rangle
    + \int_0^1 \langle \gr F(\bbx(t)) - \gr F(\bbx_k), \bbx_{k+1} - \bbx_k \rangle \diff t\\
    & \overset{(a)}{\leq} 
    F(\bbx_k) + \langle \gr F(\bbx_k), \bbx_{k+1} - \bbx_k \rangle+\int_0^1 \left(L_0+L_1\Vert \gr F(\bbx_k) \Vert \right)\Vert \bbx_{k+1} - \bbx_k \Vert^2  t\, \diff t\\
    & = F(\bbx_k) +\langle \gr F(\bbx_k), \bbx_{k+1} - \bbx_k \rangle+
    \frac{1}{2} \left(L_0+L_1\Vert \gr F(\bbx_k) \Vert \right)\Vert \bbx_{k+1} - \bbx_k \Vert^2,
\end{align}
where in $(a)$ we used Definition \ref{def: L0-L1} since $\Vert \bbx(t) - \bbx_k \Vert \leq \Vert \bbx_{k+1} - \bbx_k \Vert = \eta_k \Vert \bbv_k \Vert  \leq \epsilon/L_0 \leq 1/(2L_1) \leq 1/L_1$. We can continue bounding $F(\bbx_{k+1})$ by replacing $\bbx_{k+1} - \bbx_k = - \eta_k \bbv_k$ in \eqref{eq: descent lemma 1} as follows
\begin{align}
    F(\bbx_{k+1})
    &\leq
    F(\bbx_k) - \eta_k \langle \gr F(\bbx_k), \bbv_k \rangle + \frac{1}{2} L_0 \eta_k^2 \Vert \bbv_k \Vert^2 + 
    \frac{1}{2} L_1 \eta_k^2 \Vert \bbv_k \Vert^2 \Vert \gr F(\bbx_k) \Vert \\
    &\leq
    F(\bbx_k) - \eta_k \langle \gr F(\bbx_k), \bbv_k \rangle + \frac{1}{2} L_0 \eta_k^2 \Vert \bbv_k \Vert^2 + \frac{1}{4} L_1 \eta_k^2 \Vert \bbv_k \Vert^3
    +
    \frac{1}{4} L_1 \eta_k^2 \Vert \bbv_k \Vert \cdot \norm{\bbv_k - \gr F(\bbx_k)}^2 \\
    &\leq
    F(\bbx_k) - 
    \frac{1}{2} \eta_k \left( 1 - L_0 \eta_k \right) \norm{\bbv_k}^2 + 
    \frac{1}{4} L_1 \eta_k^2 \Vert \bbv_k \Vert^3
    +
    \frac{1}{2} \eta_k \norm{\bbv_k - \gr F(\bbx_k)}^2 \\
    &\quad +
    \frac{1}{4} L_1 \eta_k^2 \Vert \bbv_k \Vert \cdot \norm{\bbv_k - \gr F(\bbx_k)}^2 \\
    &\overset{(b)}{\leq} 
    F(\bbx_k) - 
    \frac{1}{4} \eta_k \norm{\bbv_k}^2 + 
    \frac{1}{8} \eta_k \Vert \bbv_k \Vert^2 
    +
    \frac{1}{2} \eta_k \norm{\bbv_k - \gr F(\bbx_k)}^2
    +
    \frac{1}{8} \eta_k \norm{\bbv_k - \gr F(\bbx_k)}^2 \\
    &=
    F(\bbx_k) - 
    \frac{1}{8} \eta_k \norm{\bbv_k}^2
    +
    \frac{5}{8} \eta_k \norm{\bbv_k - \gr F(\bbx_k)}^2.
\end{align}
In deriving $(b)$ above, we particularly used the conditions $L_0 \eta_k \leq 1/2$ and $\eta_k \Vert \bbv_k \Vert \leq \epsilon/L_0 \leq 1/(2L_1)$ on the step-size.

\subsection{Proof of Lemma \ref{lemma: epsilon2 bound}} \label{sec: proof epsilon2 bound lemma}

Consider iterate $k=0,1,\cdots$ and let us denote by $k_0 \leq k$ the most recent iterate to $k$ for which $q$ divides $k_0$, that is, $k_0 = \lfloor k/q \rfloor \cdot q$. This implies that {\Spider} updates $\bbv_{k_0} = \gr f(\bbx_{k_0}; \ccalS_1)$ (Algorithm \ref{alg:Spider}). Therefore, for $k=k_0$, we have that
\begin{align}
    \E_{k_0} \Big[ \norm{\bbv_{k_0} - \gr F(\bbx_{k_0})}^2 \Big]
    =
    \E_{k_0} \Big[ \norm{\gr f(\bbx_{k_0}; \ccalS_1) - \gr F(\bbx_{k_0})}^2 \Big]
    \leq
    \frac{\sigma^2}{S_1}
    =
    \frac{\epsilon^2}{4}.
\end{align}
For any $k_0 \leq k < k_0 + q$, we have
\begin{align} \label{eq: epsilon2 bound 1}
    \E \Big[ \norm{\bbv_{k+1} - \gr F(\bbx_{k+1})}^2 | \, \ccalF_{k+1} \Big]
    &=
    \E \Big[ \norm{\gr f(\bbx_{k+1}; \ccalS_2) - \gr f(\bbx_{k}; \ccalS_2) + \bbv_k - \gr F(\bbx_{k+1})}^2 | \, \ccalF_{k+1} \Big]\\
    &=
    \E \Big[ \norm{\gr f(\bbx_{k+1}; \ccalS_2) - \gr f(\bbx_{k}; \ccalS_2) + \gr F(\bbx_{k}) - \gr F(\bbx_{k+1})}^2 | \, \ccalF_{k+1} \Big]\\
    &\quad+
    \norm{\bbv_k - \gr F(\bbx_{k})}^2.
\end{align}


Noting that the mini-batch $\ccalS_2$ is of size $|\ccalS_2| = S_2$, the first term in RHS of above can be bounded as follows
\begin{align} \label{eq: epsilon2 bound 2}
    &\quad \E \Big[ \norm{\gr f(\bbx_{k+1}; \ccalS_2) - \gr f(\bbx_{k}; \ccalS_2) + \gr F(\bbx_{k}) - \gr F(\bbx_{k+1})}^2 | \, \ccalF_{k+1} \Big] \\
    &\leq
    \frac{1}{S_2} \E \Big[ \norm{\gr f(\bbx_{k+1}; \xi) - \gr f(\bbx_{k}; \xi)}^2 | \, \ccalF_{k+1}\Big]\\
    &\leq
    \frac{1}{S_2} \left( L_0 + L_1 \norm{\gr F(\bbx_{k})} \right)^2 \norm{\bbx_{k+1} - \bbx_{k}}^2\\
    &\leq
    \frac{2}{S_2} L_0 \eta_k^2 \norm{\bbv_k}^2
    +
    \frac{2}{S_2} L_1^2 \eta_k^2 \norm{\bbv_k}^2 \cdot \norm{\gr F(\bbx_{k})}^2 \\
    &\leq
    \frac{2}{S_2} L_0^2 \eta_k^2 \norm{\bbv_k}^2 
    +
    \frac{4}{S_2} L_1^2 \eta_k^2 \norm{\bbv_k}^4
    +
    \frac{4}{S_2} L_1^2 \eta_k^2 \norm{\bbv_k}^2 \cdot \norm{\bbv_k - \gr F(\bbx_{k})}^2 \\
    &\leq
    \frac{6}{S_2} \epsilon^2
    +
    \frac{4}{S_2} \left( \frac{L_1}{L_0} \right)^2 \epsilon^2 \norm{\bbv_k - \gr F(\bbx_{k})}^2.
\end{align}
In deriving the last inequality above, we used the facts that $L_0 \eta_k \Vert \bbv_k \Vert \leq \epsilon$ and $L_1 \eta_k \Vert \bbv_k \Vert^2 \leq \epsilon$. Putting \eqref{eq: epsilon2 bound 1} and \eqref{eq: epsilon2 bound 2} together yields that
\begin{align}
    \E \Big[ \norm{\bbv_{k+1} - \gr F(\bbx_{k+1})}^2 | \, \ccalF_{k+1} \Big]
    &\leq
    \left( 1 + \frac{4}{S_2} \left( \frac{L_1}{L_0} \right)^2 \epsilon^2 \right) \norm{\bbv_k - \gr F(\bbx_{k})}^2
    +
    \frac{6}{S_2} \epsilon^2.
\end{align}
Let us take expectation from both sides of the above inequality w.r.t all the sources of randomness contained in $\{\bbx_{k_0+1:k+1},\bbv_{k_0:k}\}$ conditioned on $\ccalF_{k_0}$. We also denote $e_k \coloneqq \E [ \Vert \bbv_k - \gr F(\bbx_{k}) \Vert^2 | \, \ccalF_{k_0}]=\E_{k_0} [ \Vert \bbv_k - \gr F(\bbx_{k}) \Vert^2]$. Therefore, we have shown that the non-negative sequence $\{e_k\}$ satisfies the following for $k_0 \leq k < k_0 + q$
\begin{align}
    e_{k+1}
    \leq
    a e_k + b
\end{align}
where
\begin{align}
    a
    =
    1 + \frac{4}{S_2} \left( \frac{L_1}{L_0} \right)^2 \epsilon^2,
    \quad
    b = \frac{6}{S_2} \epsilon^2,
    \quad \text{and} \quad
    e_{k_0} \leq \frac{\epsilon^2}{4}.
\end{align}
This implies that for $k_0 \leq k < k_0 + q$, we have
\begin{align}
    e_k
    \leq
    a^{k-k_0} e_{k_0}  + b \sum_{i=0}^{k-k_0-1} a^i
    \leq
    a^{q} e_{k_0}  + b \sum_{i=0}^{q-1} a^i
    \leq
    a^q e_{k_0} + b q a^{q}.
\end{align}
Next, we plug in the specified choices of $q$ and $S_2$ in the above inequality as stated in the following,
\begin{align}
    q = 2 \frac{L_0}{L_1} \frac{\sigma}{\epsilon},
    \quad
    S_2 = 48 \frac{L_0}{L_1} \frac{\sigma}{\epsilon}.
\end{align}
We first bound $a^q$ as follows,
\begin{align}
    a^q
    =
    \left( 1 + \frac{4}{S_2} \left( \frac{L_1}{L_0} \right)^2 \epsilon^2 \right)^q
    \leq
    \left( 1 + \frac{1}{\sigma}\left( \frac{L_1}{L_0} \frac{\epsilon}{2}\right)^3 \right)^q
    =
    \left( 1 + \frac{\tilde{\epsilon}^3}{\sigma} \right)^{\sigma/\tilde{\epsilon}}
    \leq
    \left( 1 + \frac{1}{16}\frac{\tilde{\epsilon}}{\sigma} \right)^{\sigma/\tilde{\epsilon}}
    \leq
    2,
\end{align}
where we denote $\tilde{\epsilon} = \frac{L_1}{L_0} \frac{\epsilon}{2} \leq \frac{1}{4}$ and used the fact that $(1 + x/16)^{1/x} \leq 2$ for any $x > 0$. Moreover, the other term $b q a^q$ can be bounded as follows,
\begin{align}
    b q a^{q}
    \leq
    2 b q
    =
    2 \cdot \frac{6}{S_2} \epsilon^2 \cdot q
    =
    2 \cdot \frac{1}{8} \frac{L_1}{L_0} \frac{\epsilon}{\sigma} \cdot \epsilon^2  \cdot 2 \frac{L_0}{L_1} \frac{\sigma}{\epsilon}
    =
    \frac{\epsilon^2}{2}.
\end{align}
Putting all together, we have shown that
\begin{align}
    \E_{k_0} \Big[ \norm{ \bbv_k - \gr F(\bbx_{k}) }^2 \Big]
    =
    e_k
    \leq
    a^q e_{k_0} + b q a^{q}
    \leq
    2 \frac{\epsilon^2}{4}
    +
    \frac{\epsilon^2}{2}
    =
    \epsilon^2,
\end{align}
which concludes the proof.

\subsection{Proof of Lemma \ref{lemma: epsilon2 bound finite-sum}} \label{sec: proof lemma epsilon2 bound finite-sum}

The proof follows the same steps as in Lemma \ref{lemma: epsilon2 bound}. Starting with $k = k_0$, {\Spider} computes the full-batch gradient. Therefore, $\bbv_{k_0} = \gr F(\bbx_{k_0})$ and
\begin{align}
    \E_{k_0} \Big[ \norm{v_{k_0} - \gr F(\bbx_{k_0})}^2 \Big]
    =
    0.
\end{align}
Next, for any $k_0 \leq k \leq k_0 + q$, we can employ our arguments in the proof of Lemma \ref{lemma: epsilon2 bound} (See Section \ref{sec: proof epsilon2 bound lemma}). Particularly, from \eqref{eq: epsilon2 bound 1} and \eqref{eq: epsilon2 bound 2}, we have that
\begin{align}
    \E \Big[ \norm{\bbv_{k+1} - \gr F(\bbx_{k+1})}^2 | \, \ccalF_{k+1} \Big]
    &\leq
    \left( 1 + \frac{4}{S_2} \left( \frac{L_1}{L_0} \right)^2 \epsilon^2 \right) \norm{\bbv_k - \gr F(\bbx_{k})}^2
    +
    \frac{6}{S_2} \epsilon^2.
\end{align}
Similar to the proof of Lemma \ref{lemma: epsilon2 bound}, take expectations on both sides of the above inequality w.r.t all the sources of randomness contained in $\{\bbx_{k_0+1:k+1},\bbv_{k_0:k}\}$ conditioned on $\ccalF_{k_0}$ and denote $e_k \coloneqq \E [ \Vert \bbv_k - \gr F(\bbx_{k}) \Vert^2 | \, \ccalF_{k_0}]$. Therefore, we have shown that the non-negative sequence $\{e_k\}$ satisfies the following for $k_0 \leq k < k_0 + q$
\begin{align}
    e_{k+1}
    \leq
    a e_k + b
\end{align}
where
\begin{align}
    a
    =
    1 + \frac{4}{S_2} \left( \frac{L_1}{L_0} \right)^2 \epsilon^2,
    \quad
    b = \frac{6}{S_2} \epsilon^2,
    \quad \text{and} \quad
    e_{k_0} =0.
\end{align}
This yields that
\begin{align}
    e_k
    \leq
    b \sum_{i=0}^{k-k_0-1} a^i
    \leq
    b \sum_{i=0}^{q-1} a^i
    \leq
    b q a^{q}.
\end{align}

First, we can bound $a^q$ for our choices of parameters $S_2 = 12 \sqrt{n}$ and $q = \sqrt{n}$ as follows
\begin{align}
    a^q
    =
    \left( 1 + \frac{4}{S_2} \left( \frac{L_1}{L_0} \right)^2 \epsilon^2 \right)^q
    =
    \left( 1 + \frac{1}{3q}\left( \frac{L_1}{L_0} \epsilon\right)^2 \right)^q
    =
    \left( 1 + \frac{1}{12 \sqrt{n}} \right)^{\sqrt{n}}
    \leq
    2,
\end{align}
where we used the fact that $\frac{L_1}{L_0} \epsilon \leq \frac{1}{2}$. Finally, we have for every $k_0 \leq k < k_0 + q$  that
\begin{align}
    \E_{k_0} \Big[ \norm{ \bbv_k - \gr F(\bbx_{k}) }^2 \Big]
    =
    e_k
    \leq
    b q a^{q}
    \leq
    \frac{6}{S_2} \epsilon^2 \cdot q \cdot 2
    =
    \epsilon^2,
\end{align}
which concludes the proof.

\end{document}